\documentclass[12pt]{article} 
\usepackage{amsbsy,amssymb,graphicx,amsthm}
\usepackage{verbatim}
\usepackage{natbib}
\usepackage{algorithm}
\usepackage{algpseudocode}
\usepackage{hyperref}
\usepackage{authblk}

\newcommand{\norm}[1]{\left\lVert#1\right\rVert}

\newtheorem{theorem}{Theorem}

\newtheorem{lemma}[theorem]{Lemma}
\newtheorem{proposition}[theorem]{Proposition}

\newtheorem{remark}[theorem]{Remark}

\usepackage{physics}

\newcommand{\df}{\mathrm{d}}

\newcommand{\TV}{\scriptsize\mbox{TV}}
\newcommand{\ind}{\mathbf{1}}

	\title{A phase transition in sampling from Restricted Boltzmann Machines}

\usepackage{times}

\author[$*$]{Youngwoo Kwon}
\author[$*$]{Qian Qin}
\author[$\dagger$]{Guanyang Wang}
\author[$\ddagger$]{Yuchen Wei}

\affil[$*$]{School of Statistics, University of Minnesota}
\affil[$\dagger$]{Department of Statistics. Rutgers University}
\affil[$\ddagger$]{Department of Mathematics, Rutgers University}

\begin{document}

\maketitle

\begin{abstract}%
Restricted Boltzmann Machines are a class of undirected graphical models that play a key role in deep learning and unsupervised learning. In this study, we prove a phase transition phenomenon in the mixing time of the Gibbs sampler for a one-parameter Restricted Boltzmann Machine. Specifically, the mixing time varies logarithmically, polynomially, and exponentially with the number of vertices depending on whether the parameter \( c \) is above, equal to, or below a critical value \( c_\star \approx -5.87 \). A key insight from our analysis is the link between the Gibbs sampler and a dynamical system, which we utilize to quantify the former based on the behavior of the latter. To study the critical case \( c = c_\star \), we develop a new isoperimetric inequality for the sampler's stationary distribution by showing that the distribution is nearly log-concave.
\end{abstract}

	\section{Introduction}\label{sec:intro}

Restricted Boltzmann Machines (RBMs) is a family of probability distributions over binary variables.
These binary variables are separated into  the visible and the hidden layers. This bipartite architecture facilitates efficient training algorithms and has made RBMs instrumental in diverse machine learning applications. These include feature learning \citep{coates2011analysis}, dimensionality reduction \citep{hinton2006reducing}, classification \citep{larochelle2008classification}, and collaborative filtering \citep{salakhutdinov2007restricted}, among others. Furthermore, RBMs play a crucial role in Deep Belief Networks (DBNs) \citep{hinton2006fast}. They act as fundamental building blocks that improve deep learning techniques through efficient unsupervised feature learning and layer-wise training of neural networks. The development of RBMs is regarded as a key contribution in Hinton's recognition with the Nobel Prize in Physics in 2024 \cite{hinton2024nobel}.

Central to the training of RBMs is Gibbs sampling \citep{geman1984stochastic}, a Markov Chain Monte Carlo (MCMC) method known for its simplicity and effectiveness in sampling from complex distributions. In RBMs, Gibbs sampling iteratively updates the states of hidden and visible nodes (also known as `reconstruction' in \cite{hinton2012practical}), thereby approximating the probability distribution of the network. Gibbs sampling generates samples that help estimate the gradient of the likelihood function, which is crucial for the training step. The effectiveness of Gibbs sampling, and consequently the accuracy of the gradient estimator, is controlled by its mixing time, which is the number of iterations needed for the sampling process to converge to its stationary distribution. Its significance  is emphasized in Hinton's practical guide to training these models \citep{hinton2012practical}:\textit{``...The second critical aspect is the mixing rate of the alternating Gibbs Markov chain. When this rate is exceedingly low, there will be a minimal reconstruction error, despite a significant disparity between the model's distribution and that of the data..." .}

Despite the popularity of RBMs, there is still a notable lack of theoretical studies regarding the mixing time of the Gibbs chain for RBMs.  To our best knowledge, the only existing work tailored for RBMs is \cite{tosh2016mixing}. Our study introduces a special yet natural class of RBMs parametrized by one parameter $c$, which controls the nature and amount of association between the two layers.
In this set up, the Gibbs chain is, in terms of convergence properties, equivalent to a chain $(X_t)_{t=0}^{\infty}$ with state space $\Omega_n = \{j/n: \, j = 0, \dots, n\}$, where $n$ characterizes the size of the network, such that $X_{t+1} \mid X_t = x \sim n^{-1} \mathsf{Bin}(n, \sigma(cx))$, where $\mathsf{Bin}$ is the binomial distribution, and $\sigma$ is the logistic sigmoid function.
Despite its simple form, this Markov chain exhibits an interesting phase transition phenomenon in its mixing time as $c$ varies. 
We have determined a critical value \( c_{\star} \) and shown that the mixing time scales logarithmically, polynomially, and exponentially with the number of vertices when \( c \) is greater than, equal to, and less than \( c_{\star} \), respectively. Detailed description of our main results are in Section \ref{subsec: main results}.

From a technical perspective,  our analysis includes several new ingredients, which we outline below:
\begin{itemize}
	\item Our analysis introduces a novel angle by linking the random Gibbs chain to a deterministic dynamical system. This relationship plays a crucial role thoughout our analysis, as we use the behavior of the deterministic system to predict the behavior of the Gibbs chain. For instance, the phase transition in the mixing time of the Gibbs chain corresponds to whether the fixed point of the dynamical system acts as an attractor, saddle point, or repeller.
	\item We apply the drift and contraction technique to derive logarithmic mixing times for \( c > c_{\star} \). While this technique is commonly employed in continuous spaces \cite{rosenthal1995minorization, fort2003geometric}, its application to discrete spaces has been scarce until very recently \cite{zhou2022dimension, kim2023mixing, chang2024dimension}, and our work also contributes to this development. Moreover, a common issue with the bounds obtained from these techniques is that they are often overly conservative \cite{meyn1994computable, jones2001honest, qin2021limitations}. However, our findings indicate that they can yield sharp estimates, at least  in our  context.
	\item The most challenging case arises at the critical point.
	To address this, we establish a new isoperimetric inequality (Lemma \ref{lem:iso}) for the stationary distribution. 
	While standard isoperimetric inequalities are usually derived for log-concave distributions, numerical calculations suggest that the stationary distribution of  $(X_t)_{t=0}^{\infty}$ may not be log-concave.
	To overcome this, we introduce Lemma \ref{lem:logconc}, which demonstrates that the target distribution is nearly log-concave, meaning that the logarithm of its probability mass function, after a natural continuation, has a controlled upper bound on its second derivative. 
	We then use this inequality, along with a close coupling argument, to show that our Gibbs chain has a spectral gap of \(\Omega(1/n^4)\).
\end{itemize}

\subsection{RBMs and their Gibbs sampler} \label{subsec:RBM Gibbs}
Fix two positive integers \( n, n' \), and let \( G_{n,n'} \) represent the complete bipartite graph consisting of \( n \) nodes \( (v_1, v_2, \ldots, v_n) \) on one side (called visible nodes) and \( n' \) nodes \( (h_1, h_2, \ldots, h_{n'}) \) on the other side (called hidden nodes).  Let $W$ be a  real-valued matrix of size $n\times n'$. Let $a, b$ be two real-valued vectors with length $n, n'$. A configuration $y = (y(v_1), \dots, y(v_n); y(h_1), \dots, y(h_{n'}))$ refers to the allocation of binary values \(\{0,1\}\) to each vertex in the graph \(G_{n,n'}\).
 An RBM is a family of distributions parametrized by $W,a, b$ on the space of all possible configurations: 
\begin{align}\label{eq:rbm}
		P_{W,a,b}(\{y\}) = \frac{1}{Z_{W,a,b}}\exp\{-E_{W,a,b}(y)\},
\end{align}
where the energy  $E_{W,a,b}(y)$ is defined as
\begin{align*}
	E_{W,a,b}(y) = -\sum_{i = 1}^{n} a_i y(v_i) - \sum_{j = 1}^{n'}b_{j}y(h_{j}) - \sum_{i, j}y(v_{i})W_{ij} \, y(h_j),
\end{align*}
and $Z_{W,a,b}= \sum_{y} \exp\{-E_{W,a,b}(y)\}$ is the normalizing constant. See  Figure \ref{fig:RBM} for an illustration.

\begin{figure}
	
	\centering
	
	\includegraphics[width=0.5\linewidth]{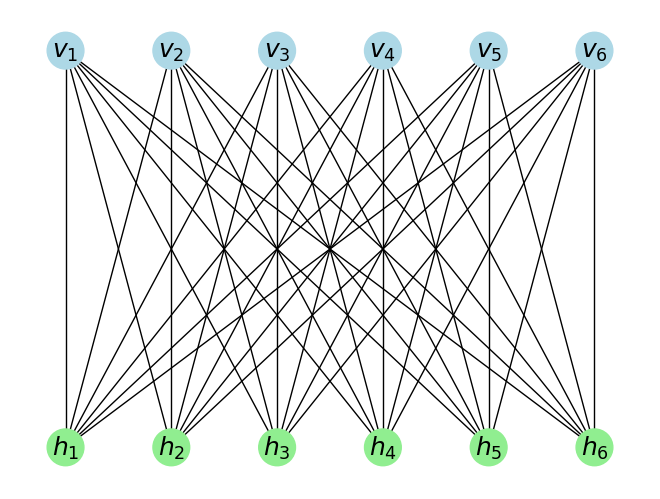}
	\caption{A restricted Boltzmann machine with six hidden nodes and six visible nodes}
	\label{fig:RBM}

\end{figure}

One method for sampling from the RBM distribution is by using the \textit{(alternating) Gibbs sampler}.  This approach involves alternately sampling the hidden and visible nodes, each based on the conditional distribution of the opposite set of nodes. Recall that the logistic sigmoid function $\sigma(x) := \exp(x)/(1+\exp(x))$,  one step the Gibbs sampler for RBM  is described in Algorithm \ref{alg:gibbs-RBM}.

\begin{algorithm}
	\caption{One step of the Gibbs sampler}
	\label{alg:gibbs-RBM}
	\begin{algorithmic}
		
		\State \textbf{Input:} Current state $$Y_t = (Y_t(v_1),\ldots, Y_t(v_n); Y_t(h_1),\ldots, Y_t(h_{n'})).$$
		\vspace{-1.5em}
		\State For  $i = 1, 2, \ldots, n$, independently draw
		\begin{align}\label{eq:visible update}
			Y_{t+1}(v_i) \sim \mathsf{Ber}\left(\sigma\left(a_i + \sum_{j=1}^{n'} W_{i,j} \, Y_t(h_j) \right) \right),
		\end{align}
		where $\mathsf{Ber}$ is the Bernoulli distribution.
		\State For  $j = 1, 2, \ldots, n'$, independently draw
		\begin{align}\label{eq:hidden update}
			Y_{t+1}(h_j) \sim \mathsf{Ber}\left(\sigma \left(b_j + \sum_{i=1}^n W_{i,j} \, Y_{t+1}(v_i) \right)\right).
		\end{align}
		\vspace{-1.5em}
		\State \textbf{Output:} Updated state $$Y_{t+1}= (Y_{t+1}(v_1),\ldots, Y_{t+1}(v_n); Y_{t+1}(h_1),\ldots, Y_{t+1}(h_{n'})).$$
	\end{algorithmic}
\end{algorithm}
	
We are interested in the convergence properties of the underlying Markov chain $(Y_t)_{t=0}^{\infty} $.


\subsection{Our setup}\label{subsec:setup}
For any matrix $A$, let $\norm{A}_{1} := \max_{1\le j\le n'}\sum_{i = 1}^{n}|W_{ij}|$. The only available work \cite{tosh2016mixing} regarding the mixing of RBMs proves the following result using a one-step coupling argument:  if $\norm{W}_{1}\norm{W^T}_{1} < 4$, then running the Gibbs sampler for 
\begin{align}\label{eqn:existing mixing}
	\frac{1}{\log(4) - \log(\norm{W}_{1}\norm{W^T}_{1})}\log\left(\frac{\min(n, n')}{\epsilon}\right)
\end{align}
iterations is sufficient to ensure that the total variation distance between the distribution of $Y_t$ and $P_{W,a,b}$ is less than $\epsilon > 0$. Therefore, the mixing time (as formally defined in Section \ref{subsec: main results}) scales logarithmically with the dimensionality, given that $\norm{W}_{1}\norm{W^T}_{1} < 4$. However, formula \eqref{eqn:existing mixing} fails to provide useful information when $\norm{W}_{1}\norm{W^T}_{1} \geq 4$. It remains unclear whether the assumption $\norm{W}_{1}\norm{W^T}_{1} < 4$ is exact or simply an artifact, and how the mixing behavior scales with dimensionality outside of the known regime.

In this work, instead of trying to be more general, we focus a special yet natural structure that allows us to obtain much more intricate results on the sampler's mixing time.

From here on, assume that $a = b = 0$, $n=n'$, and $W_{i,j} = c/n$ for all $i,j$. 
As a result, the entire family of distributions is defined by a single parameter $c\in \mathbb R$. We discuss some extensions of this model towards the end of the paper. At first glance, this approach might seem overly simplistic. However, we justify its significance for several key reasons:
\begin{itemize}
\item This framework is  analogous to the Curie-Weiss model in statistical physics, which itself is a simplified version of the Ising model. The Curie-Weiss model involves distributing $\{-1,1\}$ values across each vertex of a complete graph, where each pair of nodes has the same interaction. Similarly, our model uses a $\{0,1\}$ distribution and operates on a complete bipartite graph. The Curie-Weiss model has led to significant mathematical developments, as evidenced by beautiful works like \citep{aizenman1987rapid,ding2009mixing, levin2010glauber}, providing useful insights for more complex models.

\item Perhaps surprisingly, our subsequent analysis reveals that even in this simplified model, the mixing time undergoes a phase transition as $c$ varies. This implies that the original model could exhibit even more complex behaviors regarding its mixing time. Additionally, this suggests the impracticality of directly addressing the original model in its full generality.
\end{itemize}

For a given configuration $y $, let $s_v(y) = n^{-1} \sum_{i=1}^n y(v_i)$, and let $s_h(y) = n^{-1} \sum_{j=1}^n y(h_j)$.
Then, in the Gibbs sampler, \eqref{eq:visible update} and \eqref{eq:hidden update} become, respectively,
\begin{equation} \label{eq:deinitial}
Y_{t+1}(v_i) \sim \mathsf{Ber}(\sigma( c s_h(Y_t) )), \quad Y_{t+1}(h_j) \sim \mathsf{Ber}(\sigma( c s_v(Y_{t+1}))).
\end{equation}
(Note that $s_v(Y_{t+1})$ can be calculated without knowing the $Y_{t+1}(h_j)$'s.)

\begin{remark}
	By now it is evident that $c$ controls the association between $s_v(Y)$ and $s_h(Y)$ if $Y \sim P_{W,a,b}$.
	The two averages are positively associated when $c > 0$ and negatively associated when $c < 0$.
	The larger $|c|$ is, the stronger the association.
\end{remark}

Consider an auxiliary chain $(X_t)_{t=0}^{\infty}$, where $X_{2t} = s_h(Y_t)$ and $X_{2t+1} = s_v(Y_{t+1})$ for $t \geq 0$.
This is a time-homogeneous Markov chain such that given $X_t$, $X_{t+1}$ is distributed as $n^{-1} \mathsf{Bin}(n, \sigma(c X_t))$, where $\mathsf{Bin}$ denotes the binomial distribution.
The stationary distribution of $(X_t)_{t=0}^{\infty}$ is $P_{W,a,b} \circ s_v^{-1} = P_{W,a,b} \circ s_h^{-1}$, which is of the form
\begin{equation} \label{eq:picn}
\pi_{c,n}(\{x\}) \propto \sum_{x' \in \Omega_n } {n \choose nx} {n \choose nx'} e^{cnxx'} = {n \choose nx} (1+e^{cx})^n, \quad x \in \Omega_n,
\end{equation}
where $\Omega_n = \{j/n: \, j = 0, \dots, n\}$.

It is not difficult to see that $(Y_t)_{t=0}^{\infty}$ and $(X_t)_{t=0}^{\infty}$ have identical convergence behavior up to a scaling factor.
Indeed, since $X_{2t}$ is a function of $Y_t$, the total variation (TV) distance from the law of $X_{2t}$ to $\pi_{c,n}$ must be no greater than the TV distance from the law $Y_t$ to $P_{W,a,b}$.
On the other hand, by \eqref{eq:deinitial}, $Y_{t+1}$ depends on $Y_t$ only through $X_{2t}$.
This implies that the TV distance from the law of $Y_{t+1}$ to $P_{W,a,b}$ is no greater than the TV distance from the law of $X_{2t}$ to $\pi_{c,n}$; see \cite{roberts2001markov}, Theorem 1.
From here on, we shall focus on studying the convergence properties of the chain $(X_t)_{t=0}^{\infty}$, whose transition kernel will be denoted by $K_{c,n}$, i.e.,
\[
K_{c.n}(x, \{x'\}) = {n \choose nx'} \sigma(c x)^{nx'} [1 - \sigma(cx)]^{n-nx'}, \quad x, x' \in \Omega_n.
\]

\subsection{Main results}\label{subsec: main results}

For two probability distributions $Q_1, Q_2$ on a finite, discrete space $\Omega$, their total variation (TV) distance is given by $\|Q_1 - Q_2\|_{\scriptsize\mbox{TV}} := \sup_{A \subset \Omega} |Q_1(A) - Q_2(A)|$. 
For a Markov chain on $\Omega$ with an Mtk $K$ that has a stationary distribution $\pi$, the $\epsilon$-mixing time is defined as:
\begin{align}\nonumber
	\begin{aligned}
		\tau(\epsilon, K) :=& \max_{x\in \Omega} \inf \{t \geq 0:  \| K^t(x, \cdot) - \pi(\cdot) \|_{\scriptsize\mbox{TV}} < \epsilon\} \\
		=& \max_{x\in \Omega} \inf \{t \geq 0:  \| K^{t'}(x, \cdot) - \pi(\cdot) \|_{\scriptsize\mbox{TV}} < \epsilon \text{ for all } t' \geq t\},
	\end{aligned}
\end{align}
where the second equality follows from the well-known fact that $\| K^t(x, \cdot) - \pi(\cdot) \|_{\scriptsize\mbox{TV}}$ is a decreasing function of~$t$.
We also define $\tau(K) := \tau(1/4, K)$. It is known that $\tau(\epsilon,K) \leq \left \lceil{\log_2(1/\epsilon)}\right \rceil \tau(K)$ (Section 4.5 of \cite{levin2017markov}).

Our main results can be summarized as follows:
\begin{theorem}\label{thm:phase-transition}
	With all the settings described in Section \ref{subsec:setup}, let $x_{\star} \approx 1.278$ be the solution to the equation 
	\begin{equation} \label{eq:x0}
		\frac{x e^x}{1 + e^x} = 1,
	\end{equation}
	and let
	\begin{equation} \label{eq:c0}
		c_{\star} = -1 - x_{\star} - e^{x_{\star}} \approx -5.87.
	\end{equation}
Then each of the following holds:
\begin{itemize}
\item When $c > c_{\star}$ and $n \to \infty$,
\begin{align}\label{eq:fastmix}
\tau(K_{c,n})  = O(\log n).
\end{align}
\item When $c < c_{\star}$, there is a constant $\psi_c > 0$ such that, for $n \geq 1$,
\begin{align}\label{eq:slowmix}
	\tau(K_{c,n}) \geq \exp(\psi_c n) / 8.
	\end{align}

\item When $c = c_{\star}$, for any $\gamma \in (0,1/4)$, if $n$ is large enough,
\begin{align}\label{eq:thresholdmix}
	\tau(K_{c,n}) \geq  2 n^\gamma.
\end{align}

\item  When $c = c_{\star}$ and $n \to \infty$, 
\begin{align}\label{eq:thresholdmix2}
	\tau(K_{c,n}) = O(n^5).
\end{align}
\end{itemize}
\end{theorem}

As outlined in Theorem \ref{thm:phase-transition}, we uncover a  phase transition in the mixing time of the Gibbs sampler as \(c\) varies. 
We show that when \(c\) exceeds the critical value \(c_{\star} \approx -5.87\), the algorithm's mixing time is logarithmic relative to the size parameter \(n\), suggesting high efficiency. In contrast, if \(c\) falls below this critical value, we observe an exponential mixing time with respect to \(n\), indicative of significantly reduced efficiency. 
At the exact critical value \(c_{\star}\), we find that the sampler is polynomially mixing.

The remainder of this article is mostly devoted to establishing Theorem \ref{thm:phase-transition}.
In Section \ref{sec:general}, we describe the general strategies for proving this result.
We establish \eqref{eq:fastmix} in Section \ref{sec:bigc}, \eqref{eq:slowmix} in Section \ref{sec:smallc}, and \eqref{eq:thresholdmix} and \eqref{eq:thresholdmix2} in Section \ref{sec:criticalc}.

\section{General Strategies for Proving the Main Result} \label{sec:general}

\subsection{A dynamic system}

The first important observation is the relationship between our stochastic process and a deterministic dynamic system. Fix $c$.
Recall that given $X_t = x$, $X_{t+1}$ is distributed as $n^{-1} \mbox{Bin}(n, \sigma(c x))$, which has mean $\sigma(c x)$ and variance $\sigma(cx) [1-\sigma(cx)]/n$.
When $n$ is large, this conditional distribution is concentrated around $m_c(x) := \sigma(cx)$.

This insight suggests that we may first study the dynamical process  $x \mapsto m_c(x)$, $x \in [0,1]$. The characteristics of a dynamical system are largely influenced by its fixed (or stationary) points and the derivative of the system at these points. In our case, for any $c\in \mathbb R$,	it is easy to show that $m_c$ has a unique fixed point $x_c^* \in [0,1]$, so that $m_c(x_c^*) = x_c^*$. Since  $m_c' = c m_c (1-m_c)$, one can check the critical point $c_{\star}\approx -5.87$ is the unique value of~$c$ satisfying $|m_c'(x_c^*)| = 1$. In addition, when $c > c_{\star}$, the fixed point is attractive, while for $c < c_{\star}$, it becomes repelling. For users unfamiliar with dynamic systems, we recommend consulting \citep{martelli2011introduction} for additional references.

For a non-negative integer~$t$, we shall use $m_c^t$ to denote the composition of $t$ $m_c$'s, i.e., for $x \in [0,1]$, $m_c^0(x) = x$, and $m_c^{t+1}(x) = m_c(m_c^t(x))$.
For $\epsilon > 0$ and $x \in [0,1]$,  define the $\epsilon$-mixing time for the dynamical system $m_c$ as
\[
\hat{\tau}(\epsilon, x, c) = \inf \{ t \geq 0: \, |m_c^{t'}(x) - x_c^*| < \epsilon \text{ for all } t' \geq t \}.
\]
We then have the following result.

\begin{proposition} \label{pro:dynamic}
	Let $x \in [0,1]$ be arbitrary.
	Then each of the following holds.
	\begin{enumerate}
		\item [(i)]
		Suppose that $c>c_{\star}$.
		Then, for any $x \in [0,1]$ and an arbitrarily small $\epsilon > 0$,  $\hat{\tau}(\epsilon, x, c) = O(\log(1/\epsilon))$.
		\item [(ii)]
		Suppose that $c<c_{\star}$.
		Then, for $x \neq x_c^*$, there exists $\epsilon_x > 0$ such that $\hat{\tau}(\epsilon_x, x, c) = \infty$.
		\item [(iii)]
		Suppose that $c = c_{\star}$.
		Then,  for any $x \in [0,1]$ and $\epsilon > 0$, $\hat{\tau}(\epsilon, x, c) < \infty$.
		However, when $x \neq x_c^*$, as $\epsilon \to 0$, $\epsilon \hat{\tau}(\epsilon, x, c)$ is bounded away from 0.
	\end{enumerate}
\end{proposition}

\begin{figure}
    \begin{center}
        \includegraphics[width=\textwidth]{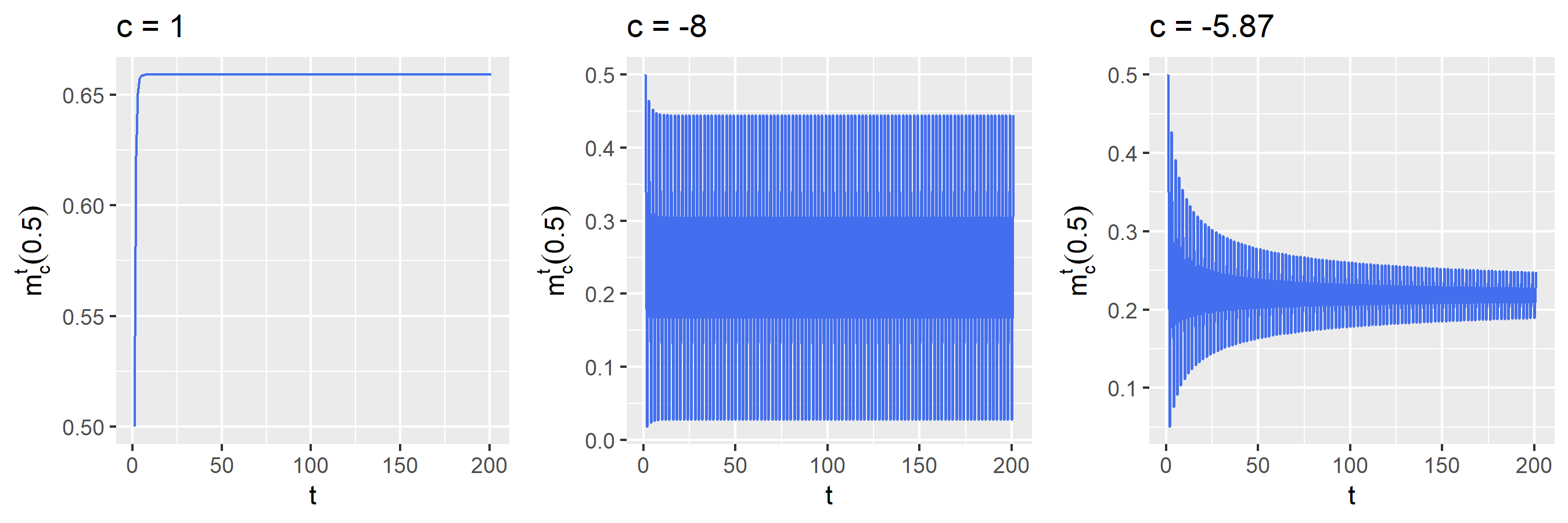}
        \includegraphics[width=\textwidth]{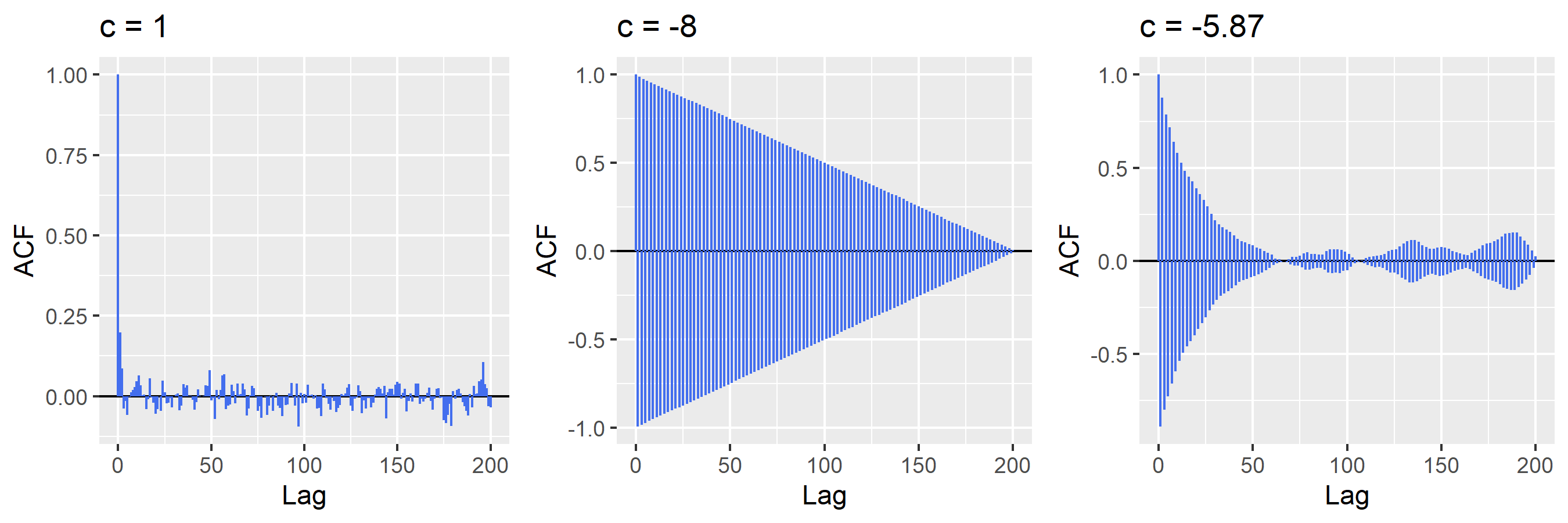}       
    \end{center}
    \caption{In the first row, $m_c^t(0.5)$ is plotted against $t$ for three values of $c$.
    The second row contains autocorrelation function plots for the Markov chain $(X_t)_t$ at three values of $c$.}
    \label{fig:sys-chain}
\end{figure}

\begin{figure}
    \begin{center}
        \includegraphics[width=\textwidth]{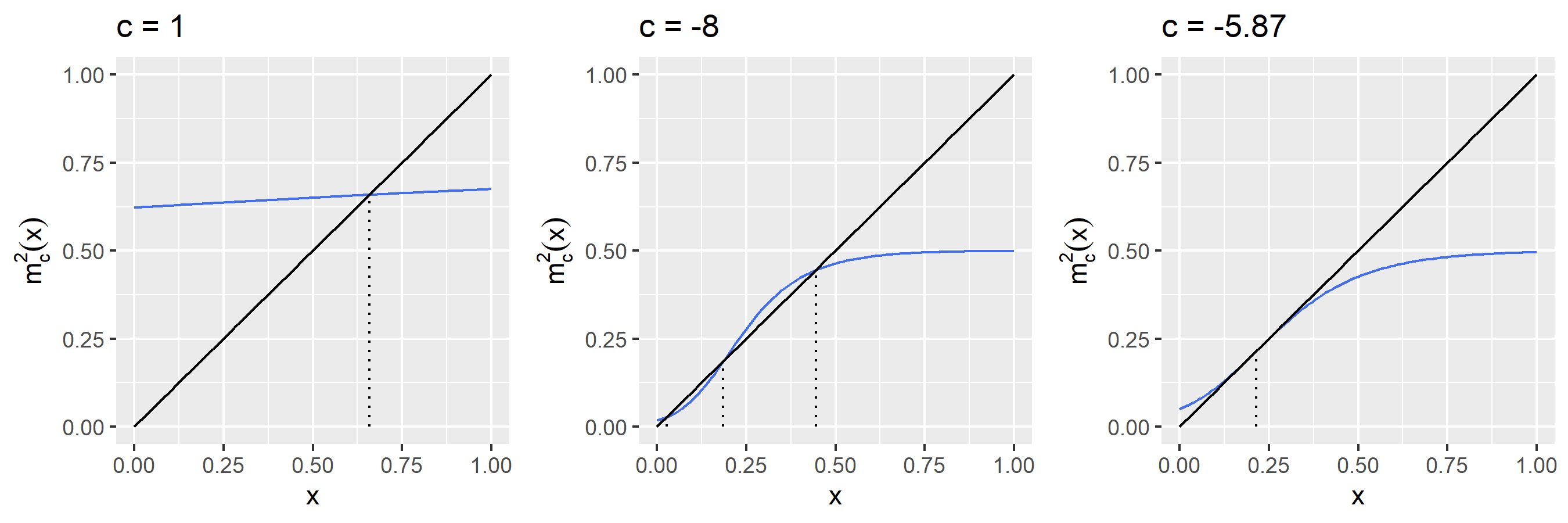}
        \includegraphics[width=\textwidth]{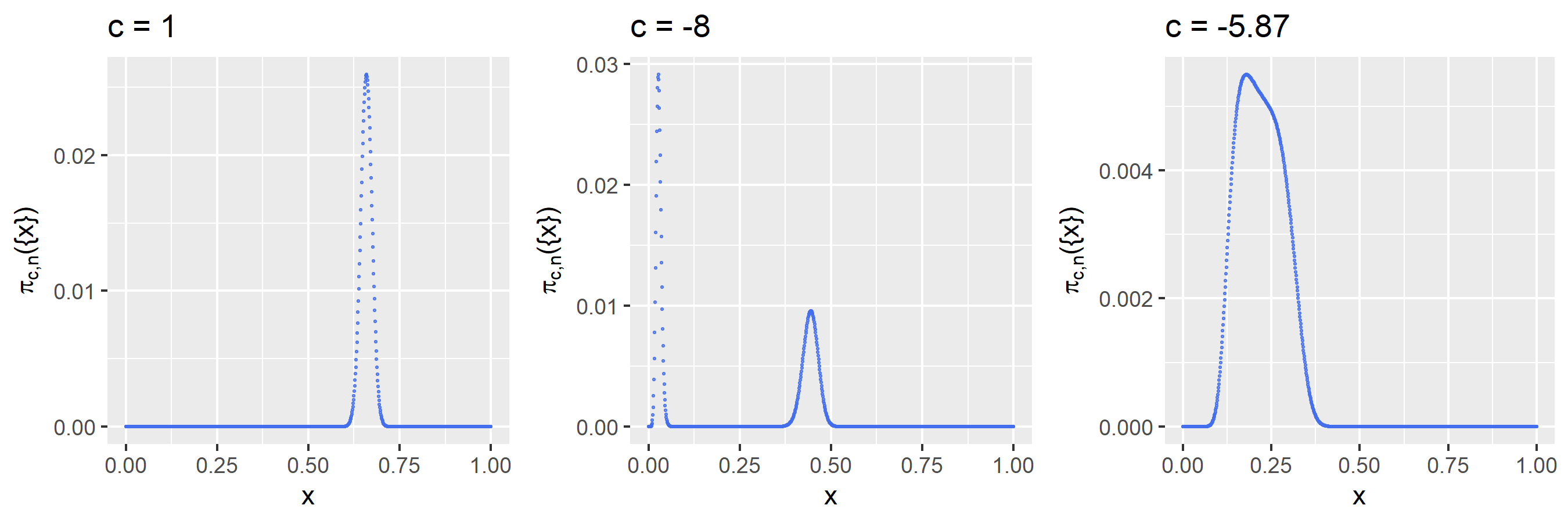}       
    \end{center}
    \caption{The first row depicts the function $x \mapsto m_c^2(x)$ at three values of $c$, alongside the $45^{\circ}$ diagonal line that passes through the origin. 
    	The second row portrays the probability mass function $x \mapsto \pi_{c,n}(\{x\})$ at three values of $c$ with $n = 1000$.}
    \label{fig:chain-stationary}
\end{figure}

Hence, the scenarios where $c > c_{\star}$, $c < c_{\star}$, and $c = c_{\star}$ correspond to three modes of convergence.
. 
Comparing Theorem \ref{thm:phase-transition} and Proposition \ref{pro:dynamic}, we see that the convergence behavior of the dynamic system translates to that of the Markov chain. See Figure \ref{fig:sys-chain} for numerical illustrations. 

Figure \ref{fig:chain-stationary} provides another perspective by drawing a connection between the stability of the fixed points of $x \mapsto m_c^2(x)$ and the shape of the stationary distribution $\pi_{c,n}$.
When $c = 1 > c_{\star}$, $x \mapsto m_c^2(x)$ has a single fixed point that is attractive; accordingly, $x \mapsto \pi_{c,n}(\{x\})$ is bell-like with steep sides, similar to a concentrated Gaussian distribution.
When $c = -8 < c_{\star}$, $x \mapsto m_c^2(x)$ has two attractive fixed points and a repelling one; correspondingly, $x \mapsto \pi_{c,n}(\{x\})$ is bimodal a bottleneck between the two modes.
When $c \approx c_{\star}$, $x \mapsto m_c^2(x)$ has a fixed point at which the derivative is approximately~1; meanwhile, $x \mapsto \pi_{c,n}(\{x\})$ is seemingly unimodal but has a broad peak.

Comparing the Markov chain $(X_t)_{t=0}^{\infty}$ with the deterministic process $(m_c^t(x))_{t=0}^{\infty}$ enables us to demonstrate how slowly $(X_t)_{t=0}^{\infty}$ converges. 
 Indeed, \eqref{eq:slowmix} and \eqref{eq:thresholdmix} are established by showing that when $n$ is large, the chain $(X_t)_{t=0}^{\infty}$ cannot mix rapidly if the corresponding dynamic system does not converge rapidly.

 We will rely on the following technical result that connects the Markov chain to the dynamic system.
 It follows from the well-known Hoeffding's inequality.
 
 \begin{lemma} \label{lem:hoeffding}
 	For $n \geq 1$, $x \in \Omega_n$, and $\epsilon > 0$,
 	\[
 	K_{c,n}(x, ( m_c(x) - \epsilon, m_c(x) + \epsilon )) \geq 1 - 2 \exp(-2n\epsilon^2).
 	\]
 	In particular, for $X \sim K_{c,n}(x; \cdot)$,
 	\[
 	E|X - m_c(x)| \leq \int_0^{\infty} 2 \exp(-2n\epsilon^2) \, \df \epsilon = \sqrt{\pi/(2n)}.
 	\]
 \end{lemma}

 \subsection{Techniques for upper bounding the mixing time}
 
 To establish upper bounds on the mixing time, i.e., \eqref{eq:fastmix} and \eqref{eq:thresholdmix2}, we employ techniques involving coupling and isoperimetric inequalities.
 
To establish \eqref{eq:fastmix}, we construct a coupling which satisfies a contraction condition if the coupled chains are both close to the fixed point $x_c^*$.
 A drift condition is then established based on properties of the dynamic system, demonstrating that the coupled chains visit the vicinity of $x_c^*$ frequently.
 Combining the drift and contraction conditions leads to a convergence bound that implies \eqref{eq:fastmix}.
 
 To establish \eqref{eq:thresholdmix2}, we first establish an isoperimetric inequality for the probability mass function of $\pi_{c_{\star},n}$, which is shown to be nearly log-concave for large $n$, and then show that this leads to $K_{c_{\star},n}$ having a spectral gap.

\section{Fast Mixing when $c > c_{\star}$} \label{sec:bigc}

When $c > c_{\star}$, we shall establish \eqref{eq:fastmix} from Theorem \ref{thm:phase-transition} using the following coupling construction.

For $x \in \Omega_n$ and $u = (u_1, \dots, u_n) \in [0,1]^n$, let 
\begin{equation} \label{eq:fcn}
f_{c,n}(x; u) = \frac{1}{n} \sum_{i=1}^n \ind_{[0, m_c(x)]}(u_i).
\end{equation}
Then $f_{c,n}(x; U)$ is distributed as $K_{c,n}(x, \cdot)$ if $U = (U_1, \dots, U_n)$ is a vector of iid $\mathsf{Uni}(0,1)$ random variables.
($\mathsf{Uni}$ denotes the uniform distribution.)
Let $W_1, W_2, \dots$ be iid copies of $U$ given above, and let $W = (W_1, W_2, \dots)$.
For $x \in \Omega_n$, let $F_{c,n}^0 (x; W) = x$, and let $F_{c,n}^{t+1}(x; W) = f_{c,n}(F_{c,n}^t(x; W); W_{n+1} )$ for $t \geq 0$.
Then, given $F_{c,n}^t(x; W) = x_t \in \Omega_n$, $F_{c,n}^{t+1}(x; W)$ is distributed as $K_{c,n}(x_t, \cdot)$.
For two arbitrary initializations $x, x' \in \Omega_n$, $(F_{c,n}^t(x; W))_{t=0}^{\infty}$ and $(F_{c,n}^t(x';W))_{t=0}^{\infty}$ form a pair of coupled chains.
We shall study the expected distance between $F_{c,n}^t(x; W)$ and $F_{c,n}^t(x'; W)$, which serves as an upper bound on the Wasserstein distance between $K_{c,n}^t(x, \cdot)$ and $K_{c,n}^t(x', \cdot)$.

The following establishes a contraction condition for this coupling construction.
Its proof is provided in Appendix \ref{app:bigc-proofs}.

\begin{lemma} \label{lem:contraction}
	Suppose that $c > c_{\star}$.
	Let $U = (U_1, \dots, U_n)$ be a vector of iid $\mathsf{Uni}(0,1)$ random variables.
	Then, there exist constants $\epsilon_c > 0$ and $\gamma_c \in [0,1)$ such that for any positive integer~$n$ and $x,x' \in \Omega_n$,
	\[
	E |f_{c,n}(x; U) - f_{c,n}(x'; U)| \leq \begin{cases}
		\min \{ |c|/4, \gamma_c\} |x-x'| & x,x' \in [x_c^*-\epsilon_c, x_c^* + \epsilon_c] \\
		(|c|/4) |x-x'| & \text{otherwise}
	\end{cases}.
	\]

\end{lemma}

Lemma \ref{lem:contraction} implies that a pair of coupled chains tend to approach each other if they are both close to the fixed point $x_c^*$.
We now give a drift condition, which loosely speaking, implies that the chains would visit the vicinity of $x_c^*$ often.
For $x \in \Omega_n$, let
\[
V_c(x) = |m_c(x) - x_c^*| + |x - x_c^*|.
\]
The following lemma is established in Appendix \ref{app:bigc-proofs}.

\begin{lemma} \label{lem:drift}
	Suppose that $c > c_{\star}$.
	Then, for each positive integer~$n$, there exist $\lambda_{c,n} < 1$ and $L_{c,n} < \infty$ such that, for $x \in \Omega_n$,
	\[
	K_{c,n} V_c(x) := \sum_{x' \in \Omega_n} K_{c,n}(x, \{x'\}) V_c(x') \leq \lambda_{c,n} V_c(x) + L_{c,n}.
	\]
	Moreover, $\lambda_{c,n}$ and $L_{c,n}$ can be chosen in a way such that $\lambda_{c,n}  = \lambda_c$ does not depend on~$n$, whereas $L_{c,n} \to 0$ as $n \to \infty$.
\end{lemma}

We can now combine the drift and contraction condition to obtain a geometric convergence result.
The proof of this result, which is given in Appendix \ref{app:bigc-proofs}, relies on existing techniques \citep[see, e.g.,][]{butkovsky2014subgeometric,douc2018markov,qin2019geometric}.
The general idea is combining Lemmas \ref{lem:contraction} and \ref{lem:drift} to show that, for $x, x' \in \Omega_n$, the expected value of some weighted distance between $F_{c,n}^t(x)$ and $F_{c,n}^t(x')$ is upper bounded by $5 \rho_c^t$ for some $\rho_c < 1$.

\begin{proposition}\label{pro:geometric-convergence}
	Suppose that $c > c_{\star}$.
	Then, there exists $\rho_c < 1$ and $n_c > 0$ such that, for any $n \geq n_c$ and $t \geq 0$,
\begin{align}\nonumber
	\begin{aligned}
	\|K_{c,n}^t(x, \cdot) - \pi_{c,n}(\cdot)\|_{\scriptsize\mbox{TV}} \leq 5n \rho_c^t.
	\end{aligned}
\end{align}
\end{proposition}

It follows from Proposition \ref{pro:geometric-convergence} that \eqref{eq:fastmix} from Theorem \ref{thm:phase-transition} holds, i.e., $\tau(K_{c,n}) = O(\log n)$ as $n \to \infty$.

\section{Slow Mixing when  $c < c_{\star}$} \label{sec:smallc}

The goal of this section is to establish \eqref{eq:slowmix} from Theorem \ref{thm:phase-transition}.
Recall (ii) of Proposition \ref{pro:dynamic}, which states that the dynamic system $x \mapsto m_c(x)$ is not convergent when $c < c_{\star}$.
By comparing the Markov chain to the dynamic system, we will establish the following result on the torpid convergence of chain.

\begin{proposition} \label{pro:exp}
	Suppose that $c < c_{\star}$. Then, one can find $\psi_c>0$ such that, for any positive integer~$n$ and $t \geq 0$,
	\begin{equation} \label{ine:exp}
	\max_{x \in \Omega_n} \|K_{c,n}^t(x,\cdot) - \pi_{c,n}(\cdot)\|_{\scriptsize\mbox{TV}} \geq 1/2 - 2t e^{-\psi_c n}.
	\end{equation}
\end{proposition}

\begin{proof}
	In Appendix \ref{app:smallc}, we establish Lemma \ref{lem:mcprime-csmall}, which states that $m_c'(x_c^*)< -1$.
	Select $\kappa > 1$ and $\epsilon \in (0,\min\{x_c^*,1-x_c^*\})$ so that $m_c'(x) \leq -\kappa$ when $x \in [x_c^* - \epsilon, x_c^* + \epsilon]$.
	Fix a positive integer~$n$.
	Let $x_- \in A_-(\epsilon) := [0,x_c^*-\epsilon]$ and $x_+ \in A_+(\epsilon) := [x_c^* + \epsilon, 1]$ be multiples of $1/n$.
	(Note that we can always let $x_- = 0$ and $x_+ = 1$.)

Using the prefixed $n$, construct two Markov chains $(X_t)$ and $(X'_t)$ that follow the transition law~$K_{c,n}$ such that $X_0 = x_-$ and $X'_0 = x_+$.
	(Whether the two chains are independent does not matter.)
	By the triangle inequality, for $t \geq 0$,
	\begin{equation} \label{ine:Kt}
		\begin{aligned}
			&\|K_{c,n}^t(x_-,\cdot) - \pi_{c,n}(\cdot) \|_{\scriptsize\mbox{TV}} + 	\|K_{c,n}^t(x_+,\cdot) - \pi_{c,n}(\cdot) \|_{\scriptsize\mbox{TV}} \\
			\geq & \| K_{c,n}^t(x_-,\cdot) - K_{c,n}^t(x_+,\cdot) \|_{\scriptsize\mbox{TV}} \\
			\geq & \begin{cases}
				K_{c,n}^t(x_-, A_-(\epsilon)) - K_{c,n}^t(x_+,A_-(\epsilon)), & t \text{ is even} \\
				K_{c,n}^t(x_-, A_+(\epsilon)) - K_{c,n}^t(x_+,A_+(\epsilon)), & t \text{ is odd}
			\end{cases} \\
			\geq & \begin{cases}
				1 - P(X_t \not\in A_-(\epsilon)) - P(X'_t \not\in A_+({\epsilon})) & t \text{ is even} \\
				1 - P(X_t \not\in A_+(\epsilon)) - P(X'_t \not\in A_-(\epsilon)), & t \text{ is odd} 
			\end{cases}.
		\end{aligned}
	\end{equation}
	
	Let
	\[
	P_t = \begin{cases}
		P(X_t \not\in A_-(\epsilon)) + P(X'_t \not\in A_+({\epsilon})), & t \text{ is even} \\
		P(X_t \not\in A_+(\epsilon)) + P(X'_t \not\in A_-(\epsilon)), & t \text{ is odd} 
	\end{cases}.
	\]
	We claim $P_t \leq 4t \exp [2n(\kappa-1)^2 \epsilon^2]$ for any integer $t \geq 0$.
	This is obviously true when $t=0$.
	Assume now that the inequality holds for some given $t$.
	Consider $P_{t+1}$.
	Assume that $t+1$ is even.
	Then
	\begin{equation} \label{ine:Pt}
		\begin{aligned}
			P_{t+1}  = & P(X_{t+1} \not\in A_-(\epsilon),  \; X_t \in A_+(\epsilon)) + P(X_{t+1} \not\in A_-(\epsilon),  \; X_t \notin A_+(\epsilon))+ \\
			&P(X'_{t+1} \not\in A_+({\epsilon}),  \; X_t' \in A_-(\epsilon))) + 
			P(X'_{t+1} \not\in A_+({\epsilon}), \; X_t' \not\in A_-(\epsilon))\\
			\leq &P(X_{t+1} \not\in A_-(\epsilon), \; X_t \in A_+(\epsilon)) + P(X_{t+1}' \not\in A_+(\epsilon), \; X_t' \in A_-(\epsilon) ) + P_t \\
			 \leq &P(X_{t+1} \not\in A_-(\epsilon), \; X_t \in A_+(\epsilon)) + P(X_{t+1}' \not\in A_+(\epsilon), \; X_t' \in A_-(\epsilon) ) + \\
			 & 4t \exp [2n(\kappa-1)^2 \epsilon^2].
		\end{aligned}
	\end{equation}
	Note that $m_c$ is decreasing, so
	if $X_t \in A_+(\epsilon)$, 
	\[
	m_c(X_t) \leq m_c(x_c^*+\epsilon) \leq m_c(x_c^*) - \kappa \epsilon = x_c^* - \kappa \epsilon.
	\]
	By Lemma~\ref{lem:hoeffding}, when $X_t \in A_+(\epsilon)$,
	\[
	\begin{aligned}
		P(X_{t+1} \not\in A_-(\epsilon) \mid X_t ) &\leq P(|X_{t+1} - m_c(X_t)| \geq (\kappa-1)\epsilon \mid X_t ) \\
		&\leq 2 \exp [-2n(\kappa-1)^2\epsilon^2].
	\end{aligned}
	\]
	This implies that
	\[
	\begin{aligned}
		P(X_{t+1} \not\in A_-(\epsilon), \; X_t \in A_+(\epsilon)) &= E[P(X_{t+1} \not\in A_-(\epsilon) \mid X_t ) \, \ind_{A_+(\epsilon)}(X_t)] \\
		&\leq 2 \exp [-2n(\kappa-1)^2\epsilon^2].
	\end{aligned}
	\]
	Similarly,
	\[
	P(X_{t+1}' \not\in A_+(\epsilon), \; X_t' \in A_-(\epsilon) ) \leq 2 \exp [-2n(\kappa-1)^2\epsilon^2].
	\]
	By~\eqref{ine:Pt},
	\[
	P_{t+1} \leq 4(t+1) \exp[-2n(\kappa-1)^2\epsilon^2].
	\]
	The same can be established when $t+1$ is odd.
	Then $P_t \leq 4t \exp[2n(\kappa-1)^2\epsilon^2]$ for all~$t$ by induction.
	Take $\psi_c = 2 (\kappa-1)^2 \epsilon^2$.
	Then, by~\eqref{ine:Kt}, for $t \geq 0$,
	\[
	\max \left\{ \|K_{c,n}^t(x_-,\cdot) - \pi_{c,n}(\cdot) \|_{\scriptsize\mbox{TV}}, \,	\|K_{c,n}^t(x_+,\cdot) - \pi_{c,n}(\cdot) \|_{\scriptsize\mbox{TV}} \right\} \geq \frac{1}{2} - 2t e^{-\psi_c n}.
	\]
	Then \eqref{ine:exp} follows.
\end{proof}

We can now establish \eqref{eq:slowmix} from Theorem \ref{thm:phase-transition}, which asserts that $\tau(1/4, K_{c,n}) \geq \exp(\psi_c n)/8 $ for $n \geq 1$.
Fix a positive integer~$n$.
By Proposition \ref{pro:exp}, there exists a point $x \in \Omega_n$ such that 
\[
\|K_{c,n}^{ \lfloor \exp(\psi_c n)/8 \rfloor }(x,\cdot) - \pi_{c,n}(\cdot)\|_{\scriptsize\mbox{TV}} \geq  \frac{1}{4},
\]
where $\lfloor \cdot \rfloor$ is the floor function.
Thus, $\tau(1/4, K_{c,n}) \geq \exp(\psi_c n)/8 $.

\section{Polynomial mixing when $c = c_{\star}$} \label{sec:criticalc}

When $c = c_{\star}$, we shall establish \eqref{eq:thresholdmix} in Section \ref{ssec:lower} and \eqref{eq:thresholdmix2} in Section \ref{ssec:upper}.

\subsection{A lower bound on the mixing time} \label{ssec:lower}

In Section \ref{ssec:lower}, we establish \eqref{eq:thresholdmix} using an argument that similar to the one used in Section \ref{sec:smallc}.
We begin with an extension of Lemma \ref{lem:hoeffding} that connects $K_{c,n}^{2t}(x, \cdot)$ to $m_c^{2t}(x)$ for $x \in \Omega_n$ and $t \geq 0$.
The proof of the result is given in Appendix \ref{app:criticalc-proofs}.

\begin{lemma} \label{lem:k2t}
	Suppose that $c = c_{\star}$.
	Let $n$ be a positive integer.
	For $t \geq 0$, $x \in \Omega_n$ and $\epsilon > 0$,
	\[
	K_{c,n}^{2t}(x, \{x': |x' - m_c^{2t}(x)| \geq \epsilon \}) \leq 4 t \exp \left[ - \frac{32n \epsilon^2}{(|c|+4)^2t^2} \right].
	\]
\end{lemma}

In Appendix \ref{app:criticalc}, we establish Lemma~\ref{lem:m2t}, which states that, if $c = c_{\star}$, then there exists an open neighborhood $N$ of $x_c^*$ such that, for $x \in N$ and $t \geq 1$, $|m_c^{2t}(x) - x_c^*| \geq |x-x_c^*|/(2t)$.
Without loss of generality, we assume that $N$ is symmetric about $x_c^*$.
Combining Lemma~\ref{lem:m2t} and Lemma~\ref{lem:k2t} we get the following:
\begin{lemma} \label{lem:s-+}
	Suppose that $c = c_{\star}$.
	Let~$n$ be sufficiently large so that $x_- = \min (\Omega_n \cap N)$ and $x_+ = \max (\Omega_n \cap N)$ satisfy $x_- < x_c^* < x_+$,
	where $N$ is defined in Lemma \ref{lem:m2t}.
	Then, for $\gamma > 0$ and $t \in [1, n^{\gamma}]$,
	\[
	K_{c,n}^{2t}(x_-, \{ x': x' \geq x_{c}^* \}) \leq 4 n^{\gamma} \exp \left[ - \frac{8n^{1-4 \gamma}  |x_- - x_{c}^*|^2}{ (|c|+4)^2} \right],
	\]
	\[
	K_{c,n}^{2t}(x_+, \{ x': x' \leq x_{c}^* \}) \leq 4 n^{\gamma} \exp \left[ - \frac{8n^{1-4 \gamma}  |x_+ - x_{c}^*|^2}{ (|c|+4)^2} \right].
	\]
\end{lemma}
\begin{proof}
	We prove the first inequality.
	The proof for the second one is analogous.
	
	Let $(X_t)$ be a Markov chain associated with $K_{c,n}$ such that $X_0 = x_-$.
	Note that since $x_{c}^*$ is the fixed point and $m_c^2$ is evidently increasing, $m_c^{2t}(x_-) \leq x_{c}^*$ for $t \geq 1$.
	A necessary condition for $X_{2t} \geq x_{c}^*$ is
	\[
	|m_c^{2t}(x_-) - x_{c}^*| < \frac{|x_- - x_{c}^*|}{2t} \text{ or } |X_{2t} - m_c^{2t}(x_-)| \geq \frac{|x_- - x_{c}^*|}{2t}.
	\] 
	By Lemma~\ref{lem:m2t}, the first inequality cannot hold.
	Thus, by Lemma~\ref{lem:k2t}, when $t \in [1, n^{\gamma}]$, where $\gamma > 0$,
	\[
	\begin{aligned}
		K_{c,n}^{2t}(x_-, \{x': x' \geq x_{c}^* \}) =& P( X_{2t} \geq x_{c}^* \mid X_0 = x_-  ) \\
		\leq & P( |X_{2t} - m_c^{2t}(x_-)| \geq |x_- - x_{c}^*|/(2t) \mid X_{0} = x_- ) \\
		\leq & 4 t \exp \left[ - \frac{8n |x_- - x_{c}^*|^2}{(|c|+4)^2t^4} \right] \\
		\leq & 4 n^{\gamma} \exp \left[ - \frac{8n^{1-4 \gamma}  |x_- - x_{c}^*|^2}{(|c|+4)^2} \right].
	\end{aligned}
	\]
\end{proof}

As an immediate consequence, we have a lower bound on the time needed for the chain to converge.
\begin{proposition} \label{pro:exp-critical}
	Suppose that $c = c_{\star}$.
	For $\gamma \in (0,1/4)$ and $n$ large enough, the following holds if $t \leq n^{\gamma}$:
	\[
	\max_{x \in \Omega_n} \|K_{c,n}^{2t}(x, \cdot) - \pi_{c,n}(\cdot)\|_{\scriptsize\mbox{TV}} \geq 1/4.
	\]
\end{proposition}
\begin{proof}
	Fix $\gamma \in (0,1/4)$.
	Let $N = (x_{c}^*-\delta, x_{c}^*+\delta)$ be the open neighborhood in Lemma \ref{lem:s-+}.
	Let $n$ be large enough so that $x_-$ and $x_+$ as defined in Lemma~\ref{lem:s-+} satisfy $|x_- - x_{c}^*| > \delta/2$ and $|x_+-x_{c}^*| > \delta/2$.
	Moreover, let $n$ be large enough so that
	\[
	8n^{\gamma} \exp \left[ - \frac{2n^{1-4\gamma} \delta^2}{(|c|+4)^2} \right] \leq 1/2.
	\]
	Then, by Lemma~\ref{lem:s-+}, for $t \in [1, \lfloor n^{\gamma} \rfloor$],
	\[
	\begin{aligned}
		&\|K_{c,n}^{2t}(x_-, \cdot) - \pi_{c,n}(\cdot)\|_{\scriptsize\mbox{TV}} + \|K_{c,n}^{2t}(x_+, \cdot) - \pi_{c,n}(\cdot)\|_{\scriptsize\mbox{TV}} \\
		\geq& \|K_{c,n}^{2t}(x_-, \cdot) - K_{c,n}^{2t}(x_+, \cdot) \|_{\scriptsize\mbox{TV}} \\
		\geq& K_{c,n}^{2t}(x_-, \{x': \, x' < x_{c}^*\}) - K_{c,n}^{2t}(x_+, \{x': \, x' < x_{c}^*\}) \\
		\geq& 1 - 4 n^{\gamma} \exp \left[ - \frac{8n^{1-4 \gamma}  |x_- - x_{c}^*|^2}{2 (|c|+4)^2} \right] - 4 n^{\gamma} \exp \left[ - \frac{8n^{1-4 \gamma}  |x_+ - x_{c}^*|^2}{2 (|c|+4)^2} \right] \\
		\geq & 1 - 8n^{\gamma} \exp \left[ - \frac{2n^{1-4\gamma} \delta^2}{(|c|+4)^2} \right] \\
		\geq & 1/2.
	\end{aligned}
	\]
	The inequality clearly also holds when $t = 0$.
	Thus, for $t \leq n^{\gamma}$, $\|K_{c,n}^{2t}(x_-, \cdot) - \pi_{c,n}(\cdot)\|_{\scriptsize\mbox{TV}} \geq 1/4$ or $\|K_{c,n}^{2t}(x_+, \cdot) - \pi_{c,n}(\cdot)\|_{\scriptsize\mbox{TV}}  \geq 1/4$.
	The desired result then follows.
\end{proof}

We can now establish \eqref{eq:thresholdmix} from Theorem \ref{thm:phase-transition}.
Take $c = c_{\star}$.
Let $\gamma \in (0,1/4)$ be arbitrary.
By Proposition \ref{pro:exp-critical}, given a large enough~$n$, there exists $x \in \Omega_n$ such that
\[
\|K_{c,n}^{ \lfloor 2n^{\gamma} \rfloor }(x, \cdot) - \pi_{c,n}(\cdot)\|_{\scriptsize\mbox{TV}} \geq 1/4.
\]
This implies that $\tau(1/4, K_{c,n}) \geq 2 n^{\gamma}$ when $n$ is sufficiently large.

\subsection{An upper bound on the mixing time} \label{ssec:upper}

\subsubsection{Tools for constructing the bound} \label{sssec:tools}
Section \ref{ssec:upper} is devoted to establishing \eqref{eq:thresholdmix2} from Theorem \ref{thm:phase-transition}.
We rely on a variety of mathematical concepts and tools that are now briefly reviewed.
Although we are interested in the case $c = c_{\star}$, all results in Section \ref{sssec:tools} hold for an arbitrary $c$.

Let $L^2(\pi_{c,n})$ be the vector space of real functions on $\Omega_n$ equipped with the inner product
\[
\langle f, g \rangle = \sum_{x \in \Omega_n} f(x) g(x) \pi_{c,n}(\{x\}).
\]
The norm of a function $f$ is $\|f\| = \sqrt{\langle f, f \rangle}$.
Let $L_0^2(\pi_{c,n}) = \{f \in L^2(\pi_{c,n}): \, \langle f, \ind \rangle = 0 \}$, where $\ind$ is the function that is constantly 1.
For $f \in L_0^2(\pi_{c,n})$, let $K_{c,n} f(\cdot) = \sum_{x \in \Omega_n} K_{c,n}(\cdot, \{x\}) f(x)$.
Then $K_{c,n}$ defines a linear operator on $L_0^2(\pi_{c,n})$.
It can be checked that $K_{c,n}$ is self-adjoint and hence, $K_{c,n}^2$ is positive semi-definite.
The convergence properties of the Markov chain associated with $K_{c,n}$ are largely dictated by the spectral gap of $K_{c,n}^2$, which is a number in $[0,1]$ given by
\[
G(K_{c,n}^2) = 1 - \sup_{f \in L^2_0(\omega), f \neq 0} \frac{\langle f,K_{c,n}^2 f \rangle}{\|f\|^2}.
\]
Indeed, the following lemma is standard.
\begin{lemma} \label{lem:critical-bound}
	Let $\mu: 2^{\Omega_n} \to [0,1]$ be a distribution.
	For $t \geq 0$, let $\mu K_{c,n}^t(\cdot) = \sum_{x \in \Omega_n} \mu(\{x\}) \, K_{c,n}^t(x, \cdot)$.
	Then
	\[
	\|\mu K_{c,n}^t(\cdot) - \pi_{c,n}(\cdot) \|_{\scriptsize\mbox{TV}} \leq \frac{1}{2} \left\| \frac{\df \mu}{\df \pi_{c,n}} - \ind \right\| [1 - G(K_{c,n}^2)]^{\lfloor t/2 \rfloor}.
	\]
\end{lemma}

To bound $\|\mu K_{c,n}^t(\cdot) - \pi_{c,n}(\cdot) \|_{\scriptsize\mbox{TV}}$ from above, we need a lower bound on $G(K_{c,n}^2)$, and this will be obtained by considering the conductance, given by
\[
\Phi(K_{c,n}^2) = \inf \left\{ \frac{\sum_{x \in A} \pi_{c,n}(\{x\}) \, K_{c,n}^2(x, A^c) }{\pi_{c,n}(A) \, \pi_{c,n}(A^c)}: \, A \subset \Omega_n, \, 0 < \pi_{c,n}(A) < 1 \right\}.
\]
The famous Cheeger's inequality \cite{lawler1988bounds} implies the following.
\begin{lemma} \citep{lawler1988bounds} \label{lem:cheeger}
	$\Phi(K_{c,n}^2)^2/8 \leq G(K_{c,n}^2) \leq \Phi(K_{c,n}^2)$.
\end{lemma}

Now our goal becomes bounding $\Phi(K_{c,n}^2)$ from below.
We shall rely on a type of isoperimetric inequality.
To be specific, we will use the following lemma, which is similar to existing results in \cite{lovasz1999hit, belloni2009computational, dwivedi2019log, andrieu2024explicit}.

\begin{lemma} \citep{qin2024convergence} \label{lem:iso-conductance}
Suppose that $K_{c,n}$ and $\pi_{c,n}$ satisfy each of the following conditions:
\begin{enumerate}
	\item There exist $ \delta \in (0,\infty)$ and $\varepsilon \in (0,1]$ such that, whenever $|x-x'| < \delta$, $\|K_{c,n}^2(x,\cdot) - K_{c,n}^2(x', \cdot)\|_{\TV} \leq 1-\epsilon$. 
	\item There exists $\kappa > 0$ such that if $S_1, S_2, S_3$ form a partition of $\Omega_n$, and $\inf_{x \in S_1, \, x' \in S_2} |x-x'| \geq \delta$, then
	\[
		\pi_{c,n}(S_3) \geq \kappa \pi_{c,n}(S_1)\pi_{c,n}(S_2).
	\]
\end{enumerate}
Then, for any $a \in (0,1)$,
\[
\begin{aligned}
    \Phi(K_{c,n}^2) \geq \varepsilon \min \left\{\frac{1-a}{2}, \frac{a^2\kappa}{4}\right\}.
\end{aligned}
\]
\end{lemma}

The first condition in Lemma \ref{lem:iso-conductance} is called a close coupling condition, while the second one is called an isoperimetric inequality.
In what follows, we establish these two conditions when $c = c_{\star}$.

\subsubsection{Close coupling}

We now establish the following close coupling condition.

\begin{lemma} \label{lem:close}
Suppose that $c = c_{\star}$, and let $n$ be a positive integer.
 Let $x,x' \in \Omega_n$ be such that $|x-x'| <  3/(2n)$. 
 Then
\[
\begin{aligned}
    \|K_{c,n}^2(x,\cdot) - K_{c,n}^2(x',\cdot)\|_{\TV} \leq 1 - e^c.
\end{aligned}
\]
\end{lemma}

\begin{proof}
	The desired result obviously holds when $x = x'$.
	Assume that $x \neq x'$ and without loss of generality assume that $x' > x$.
	Since $\Omega_n = \{j/n: \, j=0,\dots,n\}$, $x' - x = 1/n$.
	
	We now construct a coupling of $K_{c,n}^2(x, \cdot)$ and $K_{c,n}^2(x',\cdot)$.
	Let $U = (U_1, \dots, U_n)$ be a vector of $\mathsf{Uni}(0,1)$ random variables.
	Let $X_1 = f_{c,n}(x; U)$ and $X_1' = f_{c,n}(x'; U)$, where $f_{c,n}$ is defined in \eqref{eq:fcn}, i.e., $f_{c,n}(x; U) = n^{-1} \sum_{i = 1}^n \ind_{[0,m_c(x)]}(U_i)$.
	Given $(X_1, X_1') = (x_1, x_1')$, let $X_2 = X_2'$ be distributed as $K_{c,n}(x_1, \cdot)$ if $x_1 = x_1'$, and let $X_2$ be distributed as $K_{c,n}(x_1, \cdot)$ and independently let $X_2'$ be distributed as $K_{c,n}(x_1', \cdot)$ if $x_1 \neq x_1'$.
	It is easy to see that $X_2 \sim K_{c,n}^2(x, \cdot)$ and $X_2' \sim K_{c,n}^2(x', \cdot)$.
	Then, by the well-known coupling inequality,
	\[
	\|K_{c,n}^2(x,\cdot) - K_{c,n}^2(x',\cdot)\|_{\TV} \leq P(X_2 \neq X_2') \leq P(X_1 \neq X_1').
	\]
	Note that
	\[
	\begin{aligned}
		P(X_1 \neq X_1') &\leq 1 - P\left( \bigcap_{i=1}^n [ \ind_{[0,m_c(x)}(U_i) = \ind_{[0,m_c(x')}(U_i) ] \right) \\
		&= 1 - [1 - |m_c(x) - m_c(x+1/n)|]^n.
	\end{aligned}
	\]
	It can be checked that, since $c = c_{\star} < 0$, $m_c$ is decreasing and convex.
	Thus, $|m_c(x) - m_c(x+1/n)| \leq m_c(0) - m_c(1/n) = 1/2 - 1/(1+e^{-c/n})$.
	Then, by routine calculations,
	\[
	\begin{aligned}
		\|K_{c,n}^2(x,\cdot) - K_{c,n}^2(x',\cdot)\|_{\TV} &\leq 1 - \left[ 1- \left( \frac{1}{2} - \frac{1}{1+ e^{-c/n}} \right) \right]^n \leq 1 - e^c.
	\end{aligned}
	\]
\end{proof}

\subsubsection{Isoperimetric inequality} 

Next, we establish the following isoperimetric inequality.
\begin{lemma}\label{lem:iso}
Suppose that $c = c_{\star}$.
Let $n$ be greater than 20.
Then if $S_1, S_2, S_3$ form a partition of $\Omega_n$, and $\inf_{x \in S_1, \, x' \in S_2} |x-x'| \geq 3/(2n)$, then
\[
\pi_{c,n}(S_3) \geq \frac{1}{e^{100} n^2} \pi_{c,n}(S_1)\pi_{c,n}(S_2).
\]
\end{lemma}

Proving \ref{lem:iso} requires a bit of work.
First note that the probability mass function of $\pi_{c,n}$, according to \eqref{eq:picn}, is
\[
\omega_n (x) = C_n {n \choose nx} (1+ e^{cx})^n, \quad x \in \Omega_n,
\]
where $C_n$ is a normalizing constant.
It is convenient to extend $\omega_n$ to $[0,1]$ by letting
\[
\omega_n(x) = \frac{C_n \Gamma(n+1) }{\Gamma(nx + 1) \Gamma(n-nx+1)} (1+e^{cx})^n, \quad x \in [0,1],
\]
where $\Gamma$ denotes the Gamma function.
Numerical calculations reveal that $\omega_n$ is not necessarily log-concave; see Figure \ref{fig:deriv2}.
Establishing isoperimetric inequalities for distributions that are not log-concave is typically challenging.
We will rely on the following lemma, established in Appendix \ref{app:criticalc-proofs}, which roughly states that $\omega_n$ is nearly log-concave when $n$ is large.

\begin{figure}
	\centering
	\includegraphics[width=0.5\textwidth]{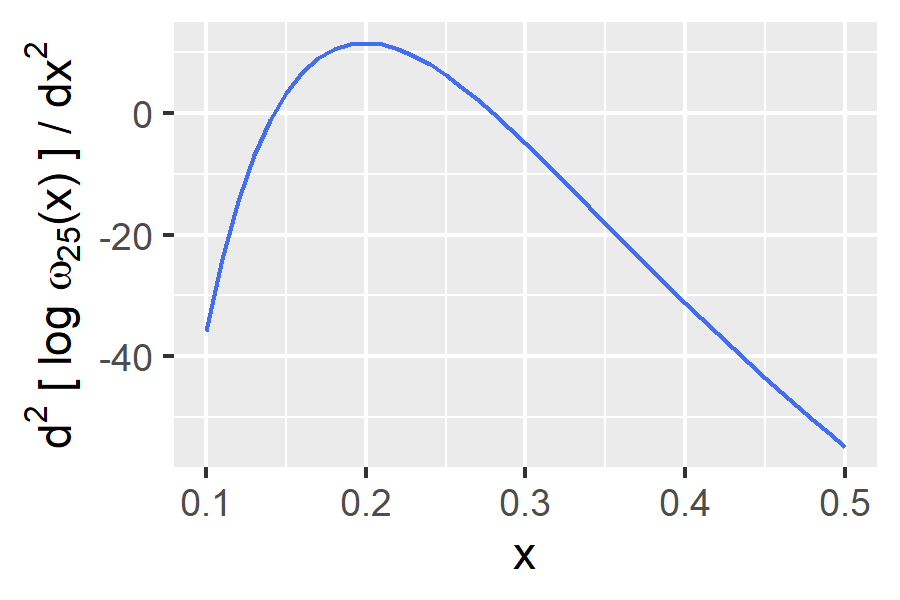}
		\caption{Plot of the second derivative of $\log \omega_{25}(x)$ for $x \in [0.1,0.5]$.} \label{fig:deriv2}
\end{figure}

\begin{lemma}\label{lem:logconc}
	Suppose that $c = c_{\star}$.
When $n > 20$, for $x \in [0,1]$,
\[
\begin{aligned}
     \frac{\df^2}{\df x^2} \log \omega_n(x)  < 800
\end{aligned}
\]
\end{lemma}

We are now ready to prove Lemma~\ref{lem:iso}.

\begin{proof}
	Let $S_1, S_2, S_3$ be a partition of $\Omega_n$ such that $\inf_{x \in S_1, \, x' \in S_2} |x-x'| \geq 3/(2n)$.
	Without loss of generality, assume that $S_1$ and $S_2$ are nonempty.
	
	Let $a_1 \in S_1$ and $a_2 \in S_2$ be such that 
	\[
	\begin{aligned}
		\omega_n(a_1) = \max_{x\in S_1} \omega_n(x), \quad \omega_n(a_2) = \max_{x\in S_2} \omega_n(x).
	\end{aligned}
	\]   
	Without loosing of generality, let $a_1 < a_2$.
	Since $\inf_{x \in S_1, \, x' \in S_2} |x-x'| \geq 3/(2n)$, there is at least one element $a_3 \in S_3$ such that $a_1 < a_3 < a_2$.
	Moreover,
	\[
	\begin{aligned}
		\pi_{c,n}(S_1)\pi_{c,n}(S_2) < n^2 \omega_n(a_1) \, \omega_n(a_2), \quad \pi_{c,n}(S_3) \geq \omega_n(a_3).
	\end{aligned}
	\]  
	Hence, it suffices to show that
	\begin{equation} \label{ine:iso-1}
	\omega_n(a_3) \geq e^{-100} \omega_n(a_1) \, \omega_n(a_2).
	\end{equation}
	
	For $x \in [0,1]$, let $g_n(x) = \log \omega_n(x) - 400 (x-a_1) (x-a_2)$.
	Then $g_n(a_1) = \log \omega_n(a_1)$, $g_n(a_2) = \log \omega_n(a_2)$.
	By Lemma \ref{lem:logconc}, $g_n$ is concave, so
	\[
	g_n(a_3) \geq \frac{a_2 - a_3}{a_2 - a_1} g_n(a_1) + \frac{a_3 - a_1}{a_2 - a_1} g_n(a_2).
	\]
	Then
	\[
	\begin{aligned}
		\log \omega_n(a_3) & \geq \frac{a_2 - a_3}{a_2 - a_1} \log \omega_n(a_1) + \frac{a_3 - a_1}{a_2 - a_1} \log \omega_n(a_2) + 400(a_3-a_1) (a_3-a_2) \\
		&\geq \log \omega_n(a_1) + \log \omega_n(a_2) - 400 \left( \frac{a_2-a_1}{2} \right)^2 \\
		&\geq \log [ \omega_n(a_1) \, \omega_n(a_2)] - 100.
	\end{aligned}
	\]
	Then \eqref{ine:iso-1} follows.
\end{proof}

\subsubsection{An upper bound on $\tau(K_{c_{\star},n})$}

Combining Lemmas \ref{lem:critical-bound}, \ref{lem:cheeger}, \ref{lem:iso-conductance}, \ref{lem:close}, and \ref{lem:iso} yields the following result.

\begin{proposition} \label{pro:bound-critical}
	Suppose that $c = c_{\star}$ and $n > 20$.
	Let $\mu: 2^{\Omega_n} \to [0,1]$ be a distribution.
	Then, for $t \geq 0$,
	\[
	\|\mu K_{c,n}^t(\cdot) - \pi_{c,n}(\cdot)\|_{\TV} \leq \frac{1}{2} \left\| \frac{\df \mu}{\df \pi_{c,n} } - \ind \right\| \left( 1 - \frac{e^{2c}}{2048 e^{200} n^4} \right)^{\lfloor t/2 \rfloor}.
	\]
\end{proposition}

\begin{proof}
	By Lemmas \ref{lem:iso-conductance}, \ref{lem:close}, and \ref{lem:iso}, $\Phi(K_{c,n}^2) \geq e^c/(16 e^{100} n^2)$.
	By Lemma \ref{lem:cheeger}, $G(K_{c,n}^2) \geq e^{2c}/(2048 e^{200} n^4)$.
	The desired result then follows from Lemma \ref{lem:critical-bound}.
\end{proof}

For $x \in \Omega_n$, let $\mu_x$ be the point mass at $x$.
Then 
\[
\begin{aligned}
	\left\| \frac{\df \mu}{\df \pi_{c,n} } - \ind \right\|^2 &= \sum_{x' \in \Omega_n} \left[ \frac{\ind_{\{x\}}(x')}{\omega_n(x')} - 1 \right]^2 \omega_n(x') \\
	&= \left[ \sum_{x' \in \Omega_n} {n \choose nx'} (1+e^{cx'})^n \right]   \left[ {n \choose nx} (1+e^{cx})^n \right]^{-1} - 1 \\
	&\leq 4^n.
\end{aligned}
\]
Then, by Proposition \ref{pro:bound-critical}, 
\[
\|K_{c,n}^t(x, \cdot) - \pi_{c,n}(\cdot) \|_{\TV} \leq 2^{n-1}  \left( 1 - \frac{e^{2c}}{2048 e^{200} n^4} \right)^{\lfloor t/2 \rfloor}.
\]
It is then easy to show that, when $n$ is sufficiently large, $\tau(1/4, K_{c,n}) = O(n^5)$.
That is, \eqref{eq:thresholdmix2} from Theorem \ref{thm:phase-transition} holds.

\section{Conclusion}
Many questions can be asked about our model and our techniques. One  intriguing question about our model is whether any of the regimes exhibit a cutoff phenomenon in the mixing time. For instance, the mixing time range for the critical case is \([c_1 n^{\gamma}, c_2 n^5]\) with \(\gamma \in (0, 1/4)\). This indicates that a more precise analysis is necessary to prove or disprove the existence of a cutoff.

Regarding extensions, a natural direction is to establish a similar phase transition for bipartite graphs with \(m\) nodes on one side and \(n\) nodes on the other, where the ratio \(m/n \rightarrow c\) as both \(m\) and \(n\) increase. Furthermore, considering the significance of non-log-concave sampling and the scarcity of existing tools \cite{balasubramanian2022towards}, we are interested in exploring additional applications of the techniques introduced in Section \ref{sec:criticalc} to analyze sampling algorithms for non-log-concave distributions.

\bigskip
\noindent{\bf Acknowledgement.}
The second author was partially supported by the National Science Foundation through grant DMS-2112887. The third author was partially supported by the National Science Foundation through grant DMS-2210849 and an Adobe Data Science Research Award.
\bibliographystyle{plain}
\bibliography{yourbibfile}

\begin{thebibliography}{10}

\bibitem{hinton2024nobel}
{Press release} for the nobel prize in physics.
\newblock https://www.nobelprize.org/prizes/physics/2024/press-release/.

\bibitem{aizenman1987rapid}
Michael Aizenman and Richard Holley.
\newblock Rapid convergence to equilibrium of stochastic ising models in the
  dobrushin shlosman regime.
\newblock In {\em Percolation Theory and Ergodic Theory of Infinite Particle
  Systems}, pages 1--11. Springer, 1987.

\bibitem{andrieu2024explicit}
Christophe Andrieu, Anthony Lee, Sam Power, and Andi~Q Wang.
\newblock {Explicit convergence bounds for Metropolis Markov chains:
  isoperimetry, spectral gaps and profiles}.
\newblock {\em The Annals of Applied Probability}, 34(4):4022--4071, 2024.

\bibitem{balasubramanian2022towards}
Krishna Balasubramanian, Sinho Chewi, Murat~A Erdogdu, Adil Salim, and Shunshi
  Zhang.
\newblock {Towards a theory of non-log-concave sampling: first-order
  stationarity guarantees for Langevin Monte Carlo}.
\newblock In {\em Conference on Learning Theory}, pages 2896--2923. PMLR, 2022.

\bibitem{belloni2009computational}
Alexandre Belloni and Victor Chernozhukov.
\newblock On the computational complexity of {MCMC}-based estimators in large
  samples.
\newblock {\em The Annals of Statistics}, pages 2011--2055, 2009.

\bibitem{butkovsky2014subgeometric}
Oleg Butkovsky.
\newblock Subgeometric rates of convergence of {M}arkov processes in the
  {W}asserstein metric.
\newblock {\em Annals of Applied Probability}, 24(2):526--552, 2014.

\bibitem{chang2024dimension}
Hyunwoong Chang and Quan Zhou.
\newblock Dimension-free relaxation times of informed {MCMC} samplers on
  discrete spaces.
\newblock arXiv preprint arXiv:2404.03867, 2024.

\bibitem{coates2011analysis}
Adam Coates, Andrew Ng, and Honglak Lee.
\newblock An analysis of single-layer networks in unsupervised feature
  learning.
\newblock In {\em Proceedings of the Fourteenth International Conference on
  Artificial Intelligence and Statistics}, pages 215--223. JMLR Workshop and
  Conference Proceedings, 2011.

\bibitem{ding2009mixing}
Jian Ding, Eyal Lubetzky, and Yuval Peres.
\newblock The mixing time evolution of {G}lauber dynamics for the mean-field
  {I}sing model.
\newblock {\em Communications in Mathematical Physics}, 289(2):725--764, 2009.

\bibitem{douc2018markov}
Randal Douc, Eric Moulines, Pierre Priouret, and Philippe Soulier.
\newblock {\em Markov Chains}.
\newblock Springer, 2018.

\bibitem{Duverney2024}
Daniel Duverney.
\newblock {\em An Introduction to Hypergeometric Functions}.
\newblock Springer, 2024.

\bibitem{dwivedi2019log}
Raaz Dwivedi, Yuansi Chen, Martin~J Wainwright, and Bin Yu.
\newblock {Log-concave sampling: Metropolis-Hastings algorithms are fast}.
\newblock {\em Journal of Machine Learning Research}, 20(183):1--42, 2019.

\bibitem{fort2003geometric}
G~Fort, E~Moulines, Gareth~O Roberts, and JS~Rosenthal.
\newblock On the geometric ergodicity of hybrid samplers.
\newblock {\em Journal of Applied Probability}, 40(1):123--146, 2003.

\bibitem{geman1984stochastic}
Stuart Geman and Donald Geman.
\newblock {Stochastic relaxation, Gibbs distributions, and the Bayesian
  restoration of images}.
\newblock {\em IEEE Transactions on Pattern Analysis and Machine Intelligence},
  PAMI-6(6):721--741, 1984.

\bibitem{hinton2012practical}
Geoffrey~E Hinton.
\newblock A practical guide to training restricted {B}oltzmann machines.
\newblock In {\em Neural Networks: Tricks of the Trade}, pages 599--619.
  Springer, 2012.

\bibitem{hinton2006fast}
Geoffrey~E Hinton, Simon Osindero, and Yee-Whye Teh.
\newblock A fast learning algorithm for deep belief nets.
\newblock {\em Neural Computation}, 18(7):1527--1554, 2006.

\bibitem{hinton2006reducing}
Geoffrey~E Hinton and Ruslan~R Salakhutdinov.
\newblock Reducing the dimensionality of data with neural networks.
\newblock {\em Science}, 313(5786):504--507, 2006.

\bibitem{jones2001honest}
Galin~L Jones and James~P Hobert.
\newblock Honest exploration of intractable probability distributions via
  {Markov chain Monte C}arlo.
\newblock {\em Statistical Science}, pages 312--334, 2001.

\bibitem{kim2023mixing}
Jungeum Kim and Veronika Rockova.
\newblock On mixing rates for {B}ayesian {CART}.
\newblock arXiv preprint arXiv:2306.00126, 2023.

\bibitem{larochelle2008classification}
Hugo Larochelle and Yoshua Bengio.
\newblock Classification using discriminative restricted {B}oltzmann machines.
\newblock In {\em Proceedings of the 25th International Conference on Machine
  Learning}, pages 536--543, 2008.

\bibitem{lawler1988bounds}
Gregory~F Lawler and Alan~D Sokal.
\newblock {Bounds on the L2 spectrum for Markov chains and Markov processes: a
  generalization of Cheeger’s inequality}.
\newblock {\em Transactions of the American Mathematical Society},
  309(2):557--580, 1988.

\bibitem{levin2010glauber}
David~A Levin, Malwina~J Luczak, and Yuval Peres.
\newblock Glauber dynamics for the mean-field ising model: cut-off, critical
  power law, and metastability.
\newblock {\em Probability Theory and Related Fields}, 146:223--265, 2010.

\bibitem{levin2017markov}
David~A Levin and Yuval Peres.
\newblock {\em Markov Chains and Mixing Times}.
\newblock American Mathematical Soc., 2nd edition, 2017.

\bibitem{lovasz1999hit}
L{\'a}szl{\'o} Lov{\'a}sz.
\newblock Hit-and-run mixes fast.
\newblock {\em Mathematical Programming}, 86:443--461, 1999.

\bibitem{martelli2011introduction}
Mario Martelli.
\newblock {\em Introduction to discrete dynamical systems and chaos}.
\newblock John Wiley \& Sons, 2011.

\bibitem{meyn1994computable}
Sean~P Meyn and Robert~L Tweedie.
\newblock Computable bounds for geometric convergence rates of {M}arkov chains.
\newblock {\em The Annals of Applied Probability}, pages 981--1011, 1994.

\bibitem{qin2024convergence}
Qian Qin.
\newblock {Convergence Bounds for Monte Carlo Markov Chains}.
\newblock arXiv preprint arXiv:2409.14656, 2024.

\bibitem{qin2021limitations}
Qian Qin and James~P Hobert.
\newblock {On the limitations of single-step drift and minorization in Markov
  chain convergence analysis}.
\newblock {\em The Annals of Applied Probability}, 31(4):1633--1659, 2021.

\bibitem{qin2019geometric}
Qian Qin and James~P Hobert.
\newblock Geometric convergence bounds for markov chains in wasserstein
  distance based on generalized drift and contraction conditions.
\newblock {\em Annales de l’Institut Henri Poincar{\'e}, Probabilit{\'e}s et
  Statistiques}, 58(2):872--889, 2022.

\bibitem{roberts2001markov}
Gareth~O Roberts and Jeffrey~S Rosenthal.
\newblock Markov chains and de-initializing processes.
\newblock {\em Scandinavian Journal of Statistics}, 28(3):489--504, 2001.

\bibitem{rosenthal1995minorization}
Jeffrey~S Rosenthal.
\newblock {Minorization conditions and convergence rates for Markov chain Monte
  Carlo}.
\newblock {\em Journal of the American Statistical Association},
  90(430):558--566, 1995.

\bibitem{salakhutdinov2007restricted}
Ruslan Salakhutdinov, Andriy Mnih, and Geoffrey Hinton.
\newblock Restricted {B}oltzmann machines for collaborative filtering.
\newblock In {\em Proceedings of the 24th International Conference on Machine
  Learning}, pages 791--798, 2007.

\bibitem{tosh2016mixing}
Christopher Tosh.
\newblock Mixing rates for the alternating {G}ibbs sampler over restricted
  boltzmann machines and friends.
\newblock In {\em International Conference on Machine Learning}, pages
  840--849. PMLR, 2016.

\bibitem{zhou2022dimension}
Quan Zhou, Jun Yang, Dootika Vats, Gareth~O Roberts, and Jeffrey~S Rosenthal.
\newblock Dimension-free mixing for high-dimensional {B}ayesian variable
  selection.
\newblock {\em Journal of the Royal Statistical Society Series B: Statistical
  Methodology}, 84(5):1751--1784, 2022.

\end{thebibliography}
\appendix

\section{Properties of the Dynamic System}

\subsection{Basic properties}

In this section, we establish some basic properties of the dynamic system $x \mapsto m_c(x)$.
In particular, we shall prove Proposition \ref{pro:dynamic}.

First recall some definitions and basic facts.
For $x \in [0,1]$ and $c \in \mathbb{R}$, $m_c(x) = \sigma(cx)$, where $\sigma(x) = e^x/(1+e^x)$.
Clearly, $m_c'(x) = c\sigma(cx)[1-\sigma(cx)]$.
When $c \geq 0$, $m_c$ is an increasing function;
when $c < 0$, $m_c$ is a decreasing function.
It is clear that $m_c$ has a fixed point $x_c^* \in [0,1]$.

\begin{lemma} \label{lem:unique}
	The fixed point is unique.
\end{lemma}
\begin{proof}
	When $c \leq 0$, $m_c$ is decreasing, so the fixed point is unique.
	Assume that $c > 0$.
	Then, for $x \in [0,1]$,
	\[
	m_c''(x) = \frac{c^2  e^{-cx} (-1+e^{-cx}) }{(1+e^{-cx})^3} \leq 0,
	\]
	and the inequality is strict if $x > 0$.
	Thus, $g_c(x) = m_c(x) - x$ is strongly concave.
	On the other hand, $g_c(0) > 0$, $g_c(1) < 0$.
	Assume that $x'$ and $x''$ are two fixed points of $m_c$, so that $0 < x' \leq x'' < 1$ and $g(x') = g(x'') = 0$.
	By the mean value theorem and strong concavity, $g'(x) < 0$ for $x \geq x'$.
	On the other hand, if $x' \neq x''$, then $g(x) = 0$ for some $x \in (x',x'')$, leading to a contradiction.
	Thus, $x' = x''$ and the proof is complete.
\end{proof}

The second order composite of $m_c$ is $x \mapsto m_c^2(x) = \sigma[c\sigma(cx)]$.
Its derivative is
\begin{equation} \label{eq:m2}
	(m_c^2)'(x) = c^2 \, \sigma'(cx) \, \sigma'(c\sigma(cx)).
\end{equation}
Note that $(m_c^2)'(x)$ is always non-negative.

Recall that $x_{\star} \approx 1.278$ be the solution to the equation 
\begin{equation} \nonumber
	\frac{x e^x}{1 + e^x} = 1,
\end{equation}
and
\begin{equation} \nonumber
	c_{\star} = -1 - x_{\star} - e^{x_{\star}} \approx -5.87.
\end{equation}
We see that
\begin{equation} \label{eq:c00}
	\frac{1}{1+e^{x_{\star}}} + \frac{x_{\star}}{c_{\star}} = \frac{c_{\star} + x_{\star} + x_{\star} e^{x_{\star}}}{(1+e^{x_{\star}}) c_{\star}} = \frac{c_{\star} + x_{\star} + 1 + e^{x_{\star}}}{(1+e^{x_{\star}})c_{\star}} = 0,
\end{equation}
where the second equality comes from~\eqref{eq:x0} and the final equality follows from~\eqref{eq:c0}.

In what follows, we give useful properties of $m_c$ in three cases: $c > c_{\star}$, $c < c_{\star}$, and $c = c_{\star}$.

\subsection{Properties of $m_c$ when $c>c_{\star}$} \label{app:bigc}

We begin with two lemmas concerning the derivatives of $m_c$ and $m_c^2$.
The first lemma shows that $x_c^*$ is attractive.

\begin{lemma} \label{lem:firstorder}
	Suppose that $c > c_{\star}$.
	Then $|m_c'(x_c^*)| < 1$.
\end{lemma}

\begin{proof}
	\noindent {\bf Case 1:} $c \geq 0$.
	
	In this case, $m_c'(x_c^*) \geq 0$.
	
	Since $c x_c^* \geq 0$,
	\[
	x_c^* = \frac{1}{1 + e^{-cx_c^*}} \geq 1/2,
	\]
	which in turn implies that
	\[
	\sigma(cx_c^*) = \frac{1}{1 + e^{-cx_c^*}} \geq \frac{1}{1 + e^{-c/2}} \geq 1/2.
	\]
	By the fact that $x \mapsto x(1-x)$ is decreasing on $[1/2,1]$,
	\[
	m_c'(x_c^*) = c \sigma(cx_c^*) [1-\sigma(cx_c^*)] \leq \frac{c e^{-c/2}}{(1+e^{-c/2})^2} \leq ce^{-c/2}.
	\]
	The quantity $ce^{-c/2}$ is maximized at $c = 2$ and the maximum is less than 1.
	This proves the claim for Case 1.
	
	\noindent {\bf Case 2:} $c \in (c_{\star},0)$. 
	
	In this case, $m_c'(x_c^*) < 0$.
	
	Let us first compare $x_c^*$ and $x_{\star}/(-c)$.
	The function
	\[
	x \mapsto m_c(x) - x = \frac{1}{1 + e^{-cx}} - x
	\]
	is strictly decreasing.
	By~\eqref{eq:c00},
	\[
	m_c(x_{\star}/(-c)) - \frac{x_{\star}}{-c} = \frac{1}{1+e^{x_{\star}}} + \frac{x_{\star}}{c} < \frac{1}{1+e^{x_{\star}}} + \frac{x_{\star}}{c_{\star}} = 0 = m_c(x_c^*) - x_c^*.
	\]
	Hence, $x_c^* < x_{\star}/(-c)$.

	Note that
	\[
	m_c'(x_c^*) = \frac{e^{-cx_c^*}}{1+e^{-cx_c^*}} \frac{c}{1 + e^{-cx_c^*}} = \frac{cx_c^* e^{-cx_c^*}}{1+e^{-cx_c^*}}.
	\]
	Note that $x \mapsto xe^x/(1+e^x)$ is strictly increasing for $x > 0$.
	By~\eqref{eq:x0} and the fact that $x_c^* < x_{\star}/(-c)$,
	\[
	\frac{-cx_c^* e^{-cx_c^*}}{1+e^{-cx_c^*}} < \frac{x_{\star} e^{x_{\star}}}{1+e^{x_{\star}}} = 1.
	\]
	Thus, $m_c'(x_c^*) > -1$, and the proof is complete.
	
\end{proof}

\begin{lemma} \label{lem:secondorder}
	Suppose that $c > c_{\star}$.
	Then $0 \leq (m_c^2)'(x) < 1$ for $x \in [0,1]$.
\end{lemma}
\begin{proof}
	We have already observed that $(m_c^2)'(x) \geq 0$ for $x \in [0,1]$.
	It suffices to show that $(m_c^2)'(x) < 1$.
	
	\noindent {\bf Case 1:} $c \geq 0$.
	
	It is easy to check that, for $x \in [0,1]$,

	\[
	\sigma'(cx) = \sigma(cx) [1 - \sigma(cx)] \in [0,1/4].
	\]
	Moreover, for $x \in [0,1]$,
	\[
	c \sigma(cx) = \frac{c}{1 + e^{-cx}} \geq c/2,
	\]
	so
	\[
	\begin{aligned}
		\sigma'(c\sigma(cx)) &= \sigma(c\sigma(cx)) [1 - \sigma(c\sigma(cx))] \\
		&\leq 1 - \sigma(c\sigma(cx)) \\
		&= \frac{1}{1 + e^{c\sigma(cx)}} \\
		& \leq \frac{1}{1 + e^{c/2}}.
	\end{aligned}
	\]
	It follows from~\eqref{eq:m2} that
	\[
	(m_c^2)'(x) \leq \frac{c^2}{4 (1 + e^{c/2})} =: g(c).
	\]
	We see that
	\[
	g'(c) = \frac{c(4+4e^{c/2}-ce^{c/2})}{8(1+e^{c/2})^2}.
	\]
	Clearly, $g'(c) > 0$ when $c \leq 4$, and $g'(c) < 0$ when $c \geq 5$.
	Hence, $g(c)$ is maximized when $c \in [4,5]$.
	It follows that
	\[
	g(c) \leq \frac{5^2}{4(1+e^{4/2})} < 1.
	\]
	Thus, $(m_c^2)'(x) < 1$.
	
	\noindent {\bf Case 2:} $c \in [-2,0)$.
	
	By~\eqref{eq:m2},
	\[
	(m_c^2)'(x) = c^2 \sigma(cx) [1-\sigma(cx)] \sigma(c\sigma(cx)) [1-\sigma(c\sigma(cx))] \leq \frac{c^2}{16} < 1.
	\]
	
	\noindent {\bf Case 3:} $c \in (c_{\star},-2)$.
	
	Fix $x \in [0,1]$.
	Let $y = -c\sigma(cx)$.
	Then $y \geq 0$.
	By~\eqref{eq:m2},
	\begin{equation} \label{eq:mab}
		\begin{aligned}
			(m_c^2)'(x) &= c^2 \sigma(cx) [1-\sigma(cx)] \sigma(c\sigma(cx)) [1-\sigma(c\sigma(cx))] \\
			&= \frac{y}{1+e^y} \frac{-c-y}{1+e^{-y}} \\
			&= a(y) b(y; c),
		\end{aligned}
	\end{equation}
	where $a(z) = z/(1+e^z)$ and $b(z; c) = (-c-z)/(1+e^{-z})$ for $z \geq 0$.
	
	We see that, for $z \geq 0$,
	\[
	a'(z) = \frac{1+e^z-ze^z}{(1+e^z)^2}.
	\]
	Observing that the numerator is a strictly decreasing function, we see that $a'(z) > 0$ when $z < x_{\star}$ and $a'(z) < 0$ when $z > x_{\star}$.
	So
	\begin{equation} \label{ine:mab1}
		a(y) \leq a(x_{\star}) = \frac{x_{\star}}{1+e^{x_{\star}}} = e^{-x_{\star}}.
	\end{equation}
	
	On the other hand, for $z \geq 0$,
	\[
	\frac{b(z;c)}{\partial z} = \frac{-e^z - (1+c+z)}{e^z (1+e^{-z})^2}.
	\]
	For $z \geq 0$ and $\gamma < -2$, define
	\[
	\beta(z;\gamma) = -e^z - (1+\gamma+z),
	\]
	so that $\beta(z;c)$ has the same sign as $\partial b(z; c) / \partial z$.
	Given $\gamma < -2$, $\beta(z; \gamma)$ is a strictly decreasing function of~$z$ on $[0,\infty)$ that is positive when $z = 0$ and negative when $z \to \infty$.
	Thus, $\beta(z;\gamma) = 0$ has exactly one positive solution $z_{\gamma}$, and this is the point at which $b(z;\gamma)$ is maximized.
	Observe that
	\begin{equation} \label{eq:bzc}
		b(z_c; c) = \frac{-c-z_c}{1+e^{-z_c}} = \frac{\beta(z_c;c) +e^{z_c}+1}{1+e^{-z_c}} = e^{z_c}.
	\end{equation}
	Let's now compare $z_c$ to $x_{\star}$.
	By~\eqref{eq:c0}, $z_{c_{\star}} = x_{\star}$.
	Since $c > c_{\star}$,
	\[
	\beta(z_{c_{\star}};c) = - e^{-z_{c_{\star}}} - (1+c+z_{c_{\star}}) < \beta(z_{c_{\star}};c_{\star}) = 0,
	\]
	which, along with the fact that $\beta(z,c)$ is a strictly decreasing function of~$z$, implies that $z_c < z_{c_{\star}} = x_{\star}$.
	Then by~\eqref{eq:bzc},
	\begin{equation} \label{ine:mab2}
		b(y; c) \leq b(z_c; c) < e^{x_{\star}}.
	\end{equation}
	
	Combining~\eqref{eq:mab},~\eqref{ine:mab1}, and~\eqref{ine:mab2}, we see that $(m_c^2)'(x) < 1$ for this case as well.
\end{proof}

We can then prove (i) of proposition \ref{pro:dynamic}, which is restated below.
Recall that 
\[
\hat{\tau}(\epsilon, x, c) = \inf \{ t \geq 0: \, |m_c^{t'}(x) - x_c^*| < \epsilon \text{ for all } t' \geq t \}.
\]
\begin{proposition}
	Suppose that $c>c_{\star}$.
	Then, for $x \in [0,1]$ and an arbitrarily small $\epsilon > 0$,  $\hat{\tau}(\epsilon, x, c) = O(\log(1/\epsilon)$.
\end{proposition}

\begin{proof}
	By Lemma \ref{lem:secondorder} and the continuity of $x \mapsto (m_c^2)'(x)$, $\sup_{x \in [0,1]} (m_c^2)'(x) < 1$.
	Then the mean value theorem shows that, for $x \in [0,1]$ and $t \geq 0$,
	\begin{equation} \label{ine:dynamic-fast-1}
	|m_c^{2t}(x) - x_c^*| = |m_c^{2t}(x) - m_c^{2t}(x_c^*)| \leq |x - x_c^*| \left[ \sup_{x' \in [0,1]} (m_c^2)'(x') \right]^t.
	\end{equation}
	On the other hand, for $x \in [0,1]$,
	\[
	|m_c'(x)| = |c \sigma(cx) [1 - \sigma(cx)] | \leq \frac{|c|}{4}. 
	\]
	Hence, for $x \in [0,1]$ and $t \geq 0$,
	\begin{equation} \label{ine:dynamic-fast-2}
	|m_c^{2t+1}(x) - x_c^*| = |m_c( m_c^{2t}(x) ) - m_c(x_c^*)| \leq \frac{|c| |m_c^{2t}(x) - x_c^*|}{4}.
	\end{equation}
	Combining \eqref{ine:dynamic-fast-1} and \eqref{ine:dynamic-fast-2} shows that, for $x \in [0,1]$ and $t \geq 0$,
	\[
	|m_c^t(x) - x_c^*| \leq \max \left\{ \frac{|c|}{4}, 1 \right\} |x-x_c^*| \left[ \sup_{x' \in [0,1]} (m_c^2)'(x') \right]^{\lfloor t/2 \rfloor}.
	\]
	The desired result then follows.
\end{proof}

\subsection{Properties of $m_c$ when $c < c_{\star}$} \label{app:smallc}

We first show that $x_c^*$ is repelling.

\begin{lemma} \label{lem:mcprime-csmall}
	Assume that $c < c_{\star}$.
	Then $m_c'(x_c^*) < -1$.
\end{lemma}

\begin{proof}
	Let us compare $x_c^*$ and $x_{\star}/(-c)$.
	The function
	\[
	x \mapsto m_c(x) - x = \frac{1}{1 + e^{-cx}} - x
	\]
	is strictly decreasing.
	By~\eqref{eq:c00},
	\[
	m_c(x_{\star}/(-c)) -\frac{ x_{\star}}{-c} = \frac{1}{1+e^{x_{\star}}} + \frac{x_{\star}}{c} > \frac{1}{1+e^{x_{\star}}} + \frac{x_{\star}}{c_{\star}} = 0 = m_c(x_c^*) - x_c^*.
	\]
	Hence, $x_{\star}/(-c) < x_c^*$.
	
	Recall that
	\[
	m_c'(x_c^*) = \frac{ce^{-cx_c^*}}{(1+e^{-cx_c^*})^2} = \frac{c x_c^* e^{-cx_c^*}}{1 + e^{-cx_c^*}}.
	\]
	Since $x \mapsto x e^x/(1+e^x)$ is strictly increasing, by~\eqref{eq:x0} and the fact that $x_{\star}/(-c) < x_c^*$,
	\[
	\frac{-cx_c^* e^{-cx_c^*}}{1+e^{-cx_c^*}} > \frac{x_{\star} e^{x_{\star}}}{1+e^{x_{\star}}} = 1.
	\]
	Thus, $m_c'(x_c^*) < -1$, and the proof is complete.
\end{proof}

We may now prove (ii) of Proposition \ref{pro:dynamic}, which is restated below.
\begin{proposition}
	Suppose that $c < c_{\star}$.
	Suppose that $c<c_{\star}$.
	Then, for $x \neq x_c^*$, there exists $\epsilon_x > 0$ such that $\hat{\tau}(\epsilon_x, x, c) = \infty$.
\end{proposition}
\begin{proof}
	Fix $x \in [0,1]$ such that $x \neq x_c^*$.
	By Lemma \ref{lem:mcprime-csmall} and the continuity of $m_c'$, we may find $\epsilon > 0$ and $\kappa > 1$ such that $\epsilon < \max\{x_c^*, 1-x_c^*\}$, and that $m_c'(x') \leq - \kappa$ when $x' \in [x_c^*- \epsilon, x_c^* + \epsilon]$.
	Let $\epsilon_x = \min\{ |x-x_c^*|, \epsilon \}$.
	Recall that $m_c$ is strictly decreasing.
	Then, for $x' \in [0,1]$ such that $x' < x_c^* - \epsilon_x$,
	\[
	m_c(x') - x_c^* \geq m_c(x_c^* - \epsilon_x) - m_c(x_c^*) \geq \kappa \epsilon_x > \epsilon_x,
	\]
	which implies that $m_c(x') > x_c^* + \epsilon_x$.
	Similarly, for $x' \in [0,1]$ such that $x' > x_c^* + \epsilon_x$,
	\[
	x_c^* - m_c(x') \geq m_c(x_c^*) - m_c(x' + \epsilon_x) \geq \kappa \epsilon_x > \epsilon_x,
	\]
	which implies that $m_c(x') < x_c^* - \epsilon_x$.
	By induction, for $t \geq 0$, $|m_c^t(x) - x_c^*| > \epsilon_x$.
	The desired result then follows.
\end{proof}

\subsection{Properties of $m_c$ when $c = c_{\star}$} \label{app:criticalc}

We begin with a characterization of $x_c^*$ when $c = c_{\star}$.

\begin{lemma} \label{lem:xstar-critical}
	Suppose that $c = c_{\star}$.
	Then $x_c^* = x_{\star}/(-c)$.
\end{lemma}

\begin{proof}
	By \eqref{eq:c00}, $x_{\star}/(-c) = 1/(1+e^{x_{\star}})$.
	On the other hand, $m_c(x_{\star}/(-c)) = 1/(1+e^{x_{\star}})$.
	Then $x_{\star}/(-c)$ is a fixed point of $m_c$.
	But the fixed point is unique by Lemma \ref{lem:unique}, so $x_{\star}/(-c) = x_c^*$.
\end{proof}

We now state a result regarding the derivatives of $m_c$ at $x_c$ at the fixed point.

\begin{lemma} \label{lem:mcstar-critical}
	Suppose that $c = c_{\star}$.
	then $m_{c}'(x_c^*) = -1$, and $(m_{c}^2)'(x_c^*) = 1$.
\end{lemma}

\begin{proof}
	Note that 
	\[
	m_{c}'(x_{c}^*) = \frac{c e^{-c x_{c}^*} }{(1 + e^{-c x_{c}^*})^2} =  \frac{c x_{c}^* e^{-c x_{c}^*} }{1 + e^{-c x_{c}^*}}.
	\]
	By \eqref{eq:x0}, the above is equal to $-1$.
	Moreover,
	\[
	(m_{c}^2)'(x_{c}^*) = m_{c}'( m_{c}(x_{c}^*) ) \, m_{c}'(x_{c}^*) = [m_{c}'(x_{c}^*)]^2 = 1.
	\]
\end{proof}

Next, we consider the derivative of $m_{c}(x)$ when $x \neq x_{c}^*$.

\begin{lemma} \label{lem:critical-contract}
	Suppose that $c = c_{\star}$.
	Then, if $x \neq x_{c}^*$, $0 \leq (m_{c}^2)'(x) < 1$.
\end{lemma}

\begin{proof}
	Fix $x \neq x_{c}^*$.
	It is clear that $m_c$ is a decreasing function, so $(m_c^2)'(x) \geq 0$.
	It remains to show that $(m_c^2)''(x) < 1$.
	Recall \eqref{eq:mab}:
	\begin{equation} \label{eq:critical-contract-ab}
	(m_c^2)'(x) = a(-c\sigma(cx)) \, b(-c\sigma(cx); c),
	\end{equation}
	where, for $z  \geq 0$,
	\[
	a(z) = \frac{z}{1 + e^z}, \quad b(z; c) = \frac{-c-z}{1+e^{-z}}.
	\]
	From the derivation of \eqref{ine:mab1}, it is evident that $a(z)$ has a unique maximizer $x_{\star}$, and the maximum is $e^{-x_{\star}}$.
	From the derivation of \eqref{ine:mab2}, we see that $b(z; c)$ has a unique maximizer $x_{\star}$, and the maximum is $e^{x_{\star}}$.
	Since $x \neq x_c^*$, 
	\[
	-c\sigma(cx) \neq -c \sigma(cx_c^*) = -c x_c^* = x_{\star},
	\]
	where the last equality is from Lemma \ref{lem:xstar-critical}.
	The desired result then follows from \eqref{eq:critical-contract-ab}.
\end{proof}

We can now establish the following result.

\begin{lemma} \label{lem:m2t}
	Suppose that $c = c_{\star}$.
	Then, there exists an open neighborhood~$N$ of $x_c^*$ such that for $x \in N$ and $t \geq 1$,
	\[
	|m_c^{2t}(x) - x_{c}^*| \geq \frac{|x-x_{c}^*|}{2t}.
	\]
\end{lemma}
\begin{proof}
	Let $\delta > 0$ and $\kappa > 0$ be such that $\delta < 1/\kappa$, and that, for $x \in N := (x_{c}^* - \delta, x_{c}^* + \delta)$, $|(m_{c}^2)''(x)| < \kappa$.
	Then, by Lemmas \ref{lem:mcstar-critical} and \ref{lem:critical-contract} along with a Taylor expansion, for $x \in N$,
	\begin{equation} \label{ine:m2x-xstar}
		|x-x_{c}^*| - \frac{\kappa}{2} |x-x_{c}^*|^2 \leq |m_{c}^2(x) -x_{c}^*| \leq |x - x_c^*|.
	\end{equation}
	
	Fix $x \in N$.
	The desired result holds for $t=1$ since, by \eqref{ine:m2x-xstar} and the fact $\kappa \delta < 1$,
	\[
	|m_c^2(x) - x_{c}^*| \geq \left(1 - \frac{\kappa}{2} |x-x_{c}^*| \right)|x-x_{c}^*| \geq |x-x_{c}^*|/2.
	\]
	Suppose that the result holds for some $t = t' \geq 1$.
	We now show that it holds for $t = t'+1$.
	By Lemma \ref{lem:critical-contract}, $m_c^{2t'}(x) \in N$, so
	\begin{equation} \nonumber
		\frac{|x-x_{c}^*|}{2t'} \leq |m_c^{2t'}(x) - x_{c}^*| < 1/\kappa.
	\end{equation}
	Then, by \eqref{ine:m2x-xstar},
	\[
	\begin{aligned}
		|m_c^{2(t'+1)}(x) - x_{c}^*| &= |m_c^2(m_c^{2t'}(x)) - x_{c}^*| 
		&\geq |m_c^{2t'}(x) - x_{c}^*| - \frac{\kappa}{2} [m_c^{2t'}(x) - x_{c}^*]^2.
	\end{aligned}
	\]
	By \eqref{ine:m2x-xstar} and the fact that the function $u \mapsto u - (\kappa/2)u^2$ is increasing on $(-\infty, 1/\kappa)$,
	\[
	\begin{aligned}
		|m_c^{2t'}(x) - x_{c}^*| - \frac{\kappa}{2} [m_c^{2t'}(x) - x_{c}^*]^2 &\geq \frac{|x-x_{c}^*|}{2t'}  - \frac{\kappa |x-x_{c}^*|^2}{8t'^2} \\
		&= \frac{|x-x_{c}^*|}{2t'} \left( 1 - \frac{\kappa |x-x_{c}^*|}{4t'} \right) .
	\end{aligned}
	\]
	Combining the two displays above and noting that $|x-x_c^*| < \delta < 1/\kappa$, we have
	\[
	\begin{aligned}
		|m_c^{2(t'+1)}(x) - x_{c}^*| & \geq \frac{|x-x_{c}^*|}{2t'} \left( 1 - \frac{\kappa |x-x_{c}^*|}{4t'} \right) \\
		&\geq \frac{|x-x_{c}^*|}{2t'} \left(1 - \frac{1}{4t'} \right) \\
		&\geq \frac{|x-x_{c}^*|}{2(t'+1)}.
	\end{aligned}
	\]
	
\end{proof}

We are now ready to prove (iii) of Proposition \ref{pro:dynamic}, which is restated below.
\begin{proposition}
	Suppose that $c = c_{\star}$.
	Then,  for any $x \in [0,1]$ and $\epsilon > 0$, $\hat{\tau}(\epsilon, x, c) < \infty$.
	However, when $x \neq x_c^*$, as $\epsilon \to 0$, $\epsilon \hat{\tau}(\epsilon, x, c)$ is bounded away from 0.
\end{proposition}

\begin{proof}
	By Lemma \ref{lem:critical-contract}, 
	\begin{equation} \label{ine:critical-contract}
		|m_c^2(x) - m_c^2(x')| < |x - x'| \quad \text{whenever } x \neq x'.
	\end{equation}
	This implies that $m_c^{2t}(x) \to x_c^*$ for $x \in [0,1]$ as $t \to \infty$, as we explain below.
	By the fact that $m_c^2$ is an increasing function, given $x$, $t \mapsto m_c^{2t}(x)$ is monotone, and thus must converge to a point $\tilde{m}_c(x) \in [0,1]$ as $t \to \infty$.
	This implies that $m_c^{2(t+1)}(x)$ converges to $m_c^2(\tilde{m}_c(x))$, so $m^2(\tilde{m}_c(x)) = \tilde{m}_c(x)$.
	It follows that, for $x \in [0,1]$,
	\[
	|m_c^2(\tilde{m}_c(x)) - m_c^2(x_c^*)| = |\tilde{m}_c(x) - x_c^*|.
	\]
	But by \eqref{ine:critical-contract}, $\tilde{m}_c(x) = x_c^*$ for each $x$, and so $m_c^{2t}(x) \to x_c^*$ as $t \to \infty$.
	Since $m_c$ is continuous, it holds that $m_c^{2t+1}(x) = m_c(m_c^{2t}(x)) \to x_c^*$ as $t \to \infty$.
	Therefore, for $x \in [0,1]$, $\lim_{t \to \infty} m_c^t(x) = x_c^*$.
	This implies that $\hat{\tau}(\epsilon, x, c) < \infty$ for each $\epsilon > 0$.
	
	On the other hand, by Lemma  \ref{lem:m2t}, for $x \neq x_c^*$ and $\epsilon > 0$, 
	\[
	\hat{\tau}(\epsilon, x, c) \geq \frac{|x-x_c^*|}{2 \epsilon} - 1.
	\]
	Then  $\epsilon \hat{\tau}(\epsilon, x, c)$ is bounded away from 0 as $\epsilon \to 0$.
\end{proof}

\section{Technical Proofs}

\subsection{Proofs for Section \ref{sec:bigc}} \label{app:bigc-proofs}

Throughout Appendix \ref{app:bigc-proofs}, assume that $c > c_{\star}$.

\subsubsection{Proof of Lemma \ref{lem:contraction}}

Recall that $U = (U_1, \dots, U_n)$ is a vector of iid $\mathsf{Uni}(0,1)$ random variables, and
\[
f_{c,n}(x; U) = \frac{1}{n} \sum_{i=1}^n \ind_{[0, m_c(x)]}(U_i).
\]
The lemma asserts that one can find $\epsilon_c > 0$ and $\gamma_c < 1$ such that, for $x, x' \in \Omega_n$,
\[
E |f_{c,n}(x; U) - f_{c,n}(x'; U)| \leq \begin{cases}
	\min \{ |c|/4, \gamma_c\} |x-x'| & x,x' \in [x_c^*-\epsilon_c, x_c^* + \epsilon_c] \\
	(|c|/4) |x-x'| & \text{otherwise}
\end{cases}.
\]

\begin{proof}
	Fix a positive integer~$n$ and $x, x' \in \Omega_n$.
	Without loss of generality, assume that $x \leq x'$.
	When $c \geq 0$, 
	\[
	E |f_{c,n}(x; U) - f_{c,n}(x'; U)| = E \left[ n^{-1} \sum_{i=1}^n \ind_{(m_c(x), m_c(x')]} (U_i) \right] = m_c(x') - m_c(x).
	\]
	A similar result holds when $c < 0$.
	Hence,
	\[
	E |f_{c,n}(x; U) - f_{c,n}(x'; U)| = |m_c(x) - m_c(x')|.
	\]
	
	For $x \in [0,1]$,
	\[
	|m_c'(x)| = |c| \sigma'(cx) = |c| \sigma(cx) [1 - \sigma(cx)] \leq |c|/4.
	\]
	By the mean value theorem,
	\[
	E |f_{c,n}(x; U) - f_{c,n}(x'; U)| = |m_c(x) - m_c(x')| \leq (|c|/4) |x-x'|.
	\]
	
	Lemma~\ref{lem:firstorder} in Appendix \ref{app:bigc} states that $|m_c'(x_c^*)| < 1$.
	Then there exist $\epsilon_c > 0$ and $\gamma_c < 1$ such that $m_c'(x) \leq \gamma_c$ for $x \in [x_c^* - \epsilon_c, x_c^* + \epsilon_c]$.
	Then, by the mean value theorem, for $s,s' \in [x_c^* - \epsilon_c, x_c^* + \epsilon_c]$,
	\[
	E |f_{c,n}(x; U) - f_{c,n}(x'; U)| = |m_c(x) - m_c(x')| \leq \gamma_c |x-x'|.
	\]
	
	The desired result is the combination of the two most recent displays.
\end{proof}

\subsubsection{Proof of Lemma \ref{lem:drift}}

Recall that, for $x \in \Omega_n$,
\[
V_c(x) = |m_c(x) - x_c^*| + |x - x_c^*|.
\]
Lemma \ref{lem:drift} states that there exist $\lambda_{c,n} < 1$ and $L_c < \infty$ such that, for $x \in \Omega_n$,
\[
K_{c,n} V_c(x) := \sum_{x' \in \Omega_n} K_{c,n}(x, \{x'\}) V_c(x') \leq \lambda_c V_c(x) + L_{c,n}.
\]
Moreover, $\lambda_c$ and $L_{c,n}$ can be chosen in a way such that $\lambda_c$ does not depend on~$n$, whereas $L_{c,n} \to 0$ as $n \to \infty$.

\begin{proof}
	Fix a positive integer~$n$ and $x \in \Omega_n$.
	Let $X \sim K_{c,n}(x, \cdot)$.
	Then
	\[
	\begin{aligned}
		K_{c,n} V_c(x) &= E|m_c(X) - x_c^*| + E|X - x_c^*| \\
		&\leq |m_c^2(x) - x_c^*| + E|m_c(X) - m_c^2(x)| + |m_c(s) - x_c^*| + E|X - m_c(x)|.
	\end{aligned}
	\]
	By the fact that $|m_c'(x)| \leq |c|/4$ for $x \in [0,1]$,
	\[
	K_{c,n} V_c(x) \leq |m_c^2(x) - x_c^*| + |m_c(x) - x_c^*| + (1 + |c|/4) E|X - m_c(x)|.
	\]
	By Lemma~\ref{lem:hoeffding},
	\[
	K_{c,n} V_c(x) \leq |m_c^2(x) - x_c^*| + |m_c(x) - x_c^*| + (1 + |c|/4) \sqrt{\pi/(2n)}.
	\]
	When $x = x_c^*$, we have
	\[
	K_{c,n} V_c(x) \leq L_{c,n},
	\]
	where
	\[
	L_{c,n} = (1 + |c|/4) \sqrt{\pi/(2n)}.
	\]
	By Lemma~\ref{lem:secondorder} in Appendix \ref{app:bigc}, there exists $k_c < 1$ such that $|m_c^2(x) - x_c^*| \leq k_c |x-x_c^*|$ for $x \in [0,1]$, so if $x \neq x_c^*$,
	\[
	\begin{aligned}
		K_n V_c(x) & \leq k_c |x - x_c^*| + |m_c(x) - x_c^*| + L_{c,n} \\
		& \leq \left( 1 - \frac{(1-k_c)|x-x_c^*|}{|x-x_c^*| + |m_c(x) - x_c^*|} \right) V_c(x) + L_{c,n} .
	\end{aligned}
	\]
	Since $m_c'(x) \leq |c|/4$ for $x \in [0,1]$, $|m_c(s) - x_c^*| \leq (|c|/4) |x-x_c^*|$, and we have
	\[
	K_n V_c(x) \leq \lambda_c V_c(x) + L_{c,n},
	\]
	where
	\[
	\lambda_c = 1 - \frac{1-k_c}{1 + |c|/4} < 1.
	\]
	Obviously, $\lambda_c$ is independent of~$n$, and $L_{c,n} \to 0$ as $n \to \infty$.
\end{proof}

\subsubsection{Proof of Proposition \ref{pro:geometric-convergence}}

The proposition states that  there exists $\rho_c < 1$ and $n_c > 0$ such that, whenever $n \geq n_c$, for $t \geq 0$ and $x \in \Omega_n$,
\[
\|K_{c,n}^t(x, \cdot) - \pi_{c,n}(\cdot)\|_{\scriptsize\mbox{TV}} \leq 5n \rho_c^t.
\]

\begin{proof}

	Recall that in Section \ref{sec:bigc}, we defined $W = (W_1, W_2, \dots)$, where $W_1, W_2, \dots$ are iid $n$-dimensional random vectors that consist of iid $\mathsf{uni}(0,1)$ random variables.
	For $x \in \Omega_n$, $F_{c,n}^0(x;W) = x$, and $F_{c,n}^{t+1}(x;W) = f_{c,n}(F_{c,n}^t(x; W); W_{n+1})$ for $t \geq 0$.
	This way, $(F_{c,n}^t(x; W))_{t=0}^{\infty}$ is a chain associated with $K_{c,n}$ started at $x$.
	
	Let $x, x' \in \Omega_n$ be arbitrary.
	For $t \geq 0$, let $X_t = F_{c,n}^t(x; W)$ and $X_t' = F_{c,n}^t(x'; W)$.
	Then $(X_t)_{t=0}^{\infty}$ and $(X_t')_{t=0}^{\infty}$ form a pair of coupled chains initiated at $(x,x')$.
	
	Let $\epsilon_c > 0$ and $\gamma_c \in [0,1)$ be as in Lemma~\ref{lem:contraction}.
	Without loss of generality, assume that $\gamma_c > 0$.
	Let $\lambda_c$ and $L_{c,n}$ be as in Lemma~\ref{lem:drift}.
	Let $n_c$ be a positive integer such that
	\begin{equation} \label{ine:n1}
		\frac{2L_{c,n} +1 - \lambda_c}{\epsilon_c + 1} \leq \frac{1- \lambda_c}{\epsilon_c/2 + 1}
	\end{equation}
	and
	\begin{equation} \label{ine:n2}
		\log \frac{|c|}{4} \log (2L_{n,c}+1) \leq \frac{1}{2} \log \gamma_c  \log \frac{\epsilon_c \lambda_c + 2}{\epsilon_c + 2}
	\end{equation}
	whenever $n \geq n_c$.
	In what follows, assume that $n \geq n_c$.

	Let $t$ be a non-negative integer.
	By Lemma~\ref{lem:contraction}, we have the following contraction condition:
	\[
	E[|X_{t+1} - X_{t+1}'| \mid X_t, X_t'] \leq \begin{cases}
		\min \{ |c|/4, \gamma_c\} |X_t - X_t'| & X_t, X_t' \in [x_c^* - \epsilon_c, x_c^* + \epsilon_c] \\
		(|c|/4) |X_t - X_t'| & \text{otherwise}
	\end{cases}.
	\]
	If
	\[
	V_c(x) + V_c(x') = |m_c(x) - x_c^*| + |x - x_c^*| + |m_c(x') - x_c^*| + |x' - x_c^*| \leq \epsilon_c,
	\]
	then both of $x$ and $x'$ are within $\epsilon_c$ of $x_c^*$.
	Thus, the above contraction condition implies that
	\[
	E[|X_{t+1} - X_{t+1}'| \mid X_t, X_t'] \leq \begin{cases}
		\gamma_c |X_t - X_t'| & V(X_t) + V(X_t') \leq \epsilon_c \\
		(|c|/4) |X_t - X_t'| & \text{otherwise}
	\end{cases}.
	\]
	
	By Lemma~\ref{lem:drift}, we have the following bivariate drift condition:
	\[
	\begin{aligned}
		&E[V_c(X_{t+1}) + V_c(X_{t+1}') + 1 \mid X_t, X_t'] \\
		\leq& \lambda_c [V_c(X_t) + V_c(X_t') + 1] + 2L_{c,n} + 1 - \lambda_c \\
		=& \lambda_c [V_c(X_t) + V_c(X_t') + 1] + \frac{2L_{c,n} + 1 - \lambda_c}{V_c(X_t) + V_c(X_t') + 1} [V_c(X_t) + V_c(X_t') + 1] \\
		\leq& \begin{cases}
			(2L_{n,c} + 1) [V_c(X_t) + V_c(X_t') + 1] & V(X_t) + V(X_t') \leq \epsilon_c \\
			\left[\lambda_c + \frac{2L_{c,n} + 1 - \lambda_c}{\epsilon_c + 1} \right] [V_c(X_t) + V_c(X_t') + 1] & \text{otherwise}
		\end{cases} \\
		\leq & \begin{cases}
			(2L_{n,c} + 1) [V_c(X_t) + V_c(X_t') + 1] & V(X_t) + V(X_t') \leq \epsilon_c \\
			\frac{\epsilon_c \lambda_c + 2}{\epsilon_c + 2} [V_c(X_t) + V_c(X_t') + 1] & \text{otherwise}
		\end{cases},
	\end{aligned}
	\]
	where the final inequality follows from~\eqref{ine:n1}.
	
	Combining the drift and contraction condition using H\"{o}lder's inequality, we see that, for $r \in (0,1]$,
	\[
	\begin{aligned}
		& E[ |X_{t+1} - X_{t+1}|^r [V_c(X_{t+1}) + V_c(X_{t+1}') + 1]^{1-r} \mid X_t, X_t' ] \\
		\leq &(E|X_{t+1} - X_{t+1}'| \mid X_t, X_t')^r \, \{ E[V_c(X_{t+1}) + V_c(X_{t+1}') + 1 \mid X_t, X_t'] \}^{1-r} \\
		\leq & \max \left\{ \gamma_c^r (2L_{n,c} + 1)^{1-r}, \, \left( \frac{|c|}{4} \right)^r \left(  \frac{\epsilon_c \lambda_c + 2}{\epsilon_c + 2} \right)^{1-r} \right\} |X_t - X_t'|^r [V_c(X_t) + V_c(X_t') + 1 ]^{1-r}.
	\end{aligned}
	\]
	If $|c| < 4$, let $r = 1$; otherwise, let
	\[
	r = \frac{\log (2L_{c,n}+1) - \log [(\epsilon_c\lambda_c + 2)/(\epsilon_c + 2)] }{ - \log \gamma_c + \log (|c|/4) + \log (2L_{c,n}+1) - \log [(\epsilon_c\lambda_c + 2)/(\epsilon_c + 2)] }.
	\]
	Then $r \in (0,1]$.
	Let 
	\[
	\rho_c = \max \left\{ \gamma_c^r (2L_{n,c} + 1)^{1-r}, \, \left( \frac{|c|}{4} \right)^r \left(  \frac{\epsilon_c \lambda_c + 2}{\epsilon_c + 2} \right)^{1-r} \right\}.
	\]
	Then $\rho_c = \max\{\gamma_c, |c|/4\} < 1$ if $|c| < 4$, and by~\eqref{ine:n2},
	\[
	\begin{aligned}
		\rho_c =&\max \left\{ \gamma_c^r (2L_{n,c} + 1)^{1-r}, \, \left( \frac{|c|}{4} \right)^r \left( \frac{\epsilon_c \lambda_c + 2}{\epsilon_c + 2} \right)^{1-r} \right\} \\
		=& \exp \left[ - \log \gamma_c \log \frac{\epsilon_c \lambda_c + 2}{\epsilon_c + 2} + \log \frac{|c|}{4} \log (2L_{c,n} + 1) \right] \\
		\leq & \exp \left( - \frac{1}{2} \log \gamma_c \log \frac{\epsilon_c \lambda_c + 2}{\epsilon_c + 2} \right) < 1
	\end{aligned}
	\]
	if $|c| \geq 4$.
	It follows that, for each non-negative integer~$t$,
	\[
	\begin{aligned}
		E[ |X_t - X_t'|^r [V_c(X_t) + V_c(X_t') + 1]^{1-r} ] \leq \rho_c^t \, |x - x'|^r \, [V_c(x) + V_c(x') + 1]^{1-r}
	\end{aligned}
	\]
	Since $|X_t - X_t'| \leq 1$, it follows that
	\[
	E|X_t - X_t'| \leq [V_c(x) + V_c(x') + 1] \rho_c^t \leq 5 \rho_c^t.
	\]
	
	By the Markov inequality,
	\[
	P(X_t \neq X_t') = P(|X_t - X_t'| \geq 1/n) \leq nE|X_t - X_t'|.
	\]
	Assume now that $x'$ is randomly drawn from $\pi_{c,n}$ and independent of $W$.
	Then $(X_t')_{t=0}^{\infty}$ is stationary.
	It then follows from the well-known coupling inequality that
	\[
	\|K_{c,n}^t(x, \cdot) - \pi_{c,n}(\cdot) \|_{\scriptsize\mbox{TV}} \leq 5 n \rho_c^t.
	\]
\end{proof}

\subsection{Proofs for Section \ref{sec:criticalc}} \label{app:criticalc-proofs}

Throughout Appendix \ref{app:criticalc-proofs}, assume that $c = c_{\star}$.

\subsubsection{Proof of Lemma \ref{lem:k2t}}

Let $n$ be a positive integer.
Lemma \ref{lem:k2t} asserts that, for $x \in \Omega_n$ and $t \geq 0$.
\[
K_{c,n}^{2t}(x, \{x': |x' - m_c^{2t}(x)| \geq \epsilon \}) \leq 4 t \exp \left[ - \frac{32n \epsilon^2}{(|c|+4)^2t^2} \right].
\]

\begin{proof}
	Let $(X_t)$ be a chain associated with $K_{c,n}$ started from an arbitrary point $x \in \Omega_n$.
	Using Lemma~\ref{lem:hoeffding} and the fact that $|m_c'(x')| \leq |c|/4$ for $x' \in [0,1]$, we see that, for $\epsilon > 0$,
	\begin{equation} \label{ine:K2}
	\begin{aligned}
		&K_{c,n}^2(x, \{x': \, |x' - m_c^2(x)| \geq \epsilon \}) \\
		=& P( |X_2 - m_c^2(x)| \geq \epsilon \mid X_0 = x ) \\
		\leq& P( |X_2 - m_c(X_1)| \geq 4\epsilon/(|c|+4) \mid X_0 = x ) + \\
		& P( |m_c(X_1) - m_c^2(x)| \geq |c|\epsilon/(|c|+4) \mid X_0 = x ) \\
		\leq& P( |X_2 - m_c(X_1)| \geq 4\epsilon/(|c|+4) \mid X_0 = x ) + \\
		& P( |X_1 - m_c(x)| \geq 4\epsilon/(|c|+4) \mid X_0 = x ) \\
		=& \sum_{x' \in \Omega_n} P( |X_2 - m_c(x')| \geq 4\epsilon/(|c|+4) \mid X_1 = x' ) \, K_{c,n}(x, \{x'\}) + \\
		& P( |X_1 - m_c(x)| \geq 4\epsilon/(|c|+4) \mid X_0 = x ) \\
		\leq& 4 \exp [ -32n \epsilon^2/(|c|+4)^2 ]
	\end{aligned}
	\end{equation}
	
	Fix $x \in \Omega_n$.
	The desired result clearly holds when $t = 0$.
	Fix $t \geq 1$.
	Then, for $\epsilon > 0$,
	\[
	\begin{aligned}
		&K_{c,n}^{2t}(x, \{x': \, |x' - m_c^{2t}(x)| \geq \epsilon \}) \\
		\leq & \sum_{t'=1}^t  P\left( \left|m_c^{2(t-t')}(X_{2t'}) - m_c^{2(t-t'+1)}( X_{2(t'-1)} ) \right| \geq \epsilon/t \mid X_0 = x \right) \\
		\leq & \sum_{t'=1}^t P( |X_{2t'} - m_c^2(X_{2(t'-1)}) | \geq \epsilon/t \mid X_0 = x ) \\
		=& \sum_{t'=1}^t \sum_{x' \in \Omega_n} K_{c,n}^2(x', \{x'': |x'' - m_c^2(x')| \geq \epsilon/t\}) \, K^{2(t'-1)}(x, \{x'\}) \\
		\leq & 4 t \exp \left[ - \frac{32n \epsilon^2}{(|c|+4)^2t^2} \right].
	\end{aligned}
	\]
	The third line in the display above follows from Lemma \ref{lem:critical-contract}, which implies that $m_c^2$ is non-expansive;
	the last line follows from \eqref{ine:K2}.
\end{proof}

\subsubsection{Proof of Lemma \ref{lem:logconc}}

To establish Lemma \ref{lem:logconc}, we first prove the following technical result.

\begin{lemma}\label{lem:n-infty}
	Suppose that $c = c_{\star}$.
	Then, for all $x \in (0,1)$, 
	$$\frac{c_{}^2e^{c_{}x}}{(1+e^{c_{}x})^2} -\frac{1}{1-x} - \frac{1}{x}\leq 0$$
\end{lemma}
\begin{proof}
	Let $g(x) = c_{}^2x(1-x) - e^{-c_{}x}(1+e^{c_{}x})^2$. It suffices to show that $g(x) \leq 0$ for all $x \in (0,1)$.

	By Lemma \ref{lem:xstar-critical}, $x_c^* = x_{\star}/(-c)$.
	Then, using the fact that $x_c^* = m_c(x_c^*) = 1/(1+e^{-cx_c^*})$ along with \eqref{eq:c0}, one can derive the following:
	\begin{equation} \label{eq:identity}
		1 = -c x_c^* (1- x_c^*), \; e^{c x_c^*} = -c x_c^* - 1, \; e^{-c x_c^*} = -1 - c + c x_c^*.
	\end{equation} 
	It can then be verified that
	\[
	\begin{aligned}
		& g(x_c^*) = c^2 x_c^* (1-x_c^*) - [e^{-c x_c^*} + 2 + e^{cx_c^*}] = 0, \\
		& g'(x_c^*) = c^2(1-2x_c^*) - c(e^{cx_c^*} - e^{-cx_c^*}) = 0.
	\end{aligned}
	\]
	Moreover, routine calculations show that, for $x \in [0,1]$,
	\[
	g''(x) = -2c^2 - c^2(e^{-cx} + e^{cx}) \leq 0.
	\]
	By the Taylor expansion of $g(x)$ around $x_c^*$, $g(x) \leq 0$ for $x \in (0,1)$.

\end{proof}

Recall that
\[
\omega_n(x) = \frac{C_n \Gamma(n+1) }{\Gamma(nx + 1) \Gamma(n-nx+1)} (1+e^{cx})^n, \quad x \in [0,1],
\]
Lemma \ref{lem:logconc} asserts that $\df^2 \log \omega_n(x) / \df x^2 < 800$ for $x \in [0,1]$ if $n > 20$.
We now prove this.

\begin{proof}
	For $x > 0$, let $\psi^{(1)}$ be the trigamma function, i.e., for $x > 0$,
	\[
	\psi^{(1)}(x) = \frac{\df^{2}}{\df x^{2}} \log \Gamma(x).
	\]
	Then, by the well-knwon fact that $\psi^{(1)}(x) = \sum_{k=0}^{\infty} 1/(x+k)^2 \geq 1/x$ \citep[see, e.g.,][Theorem 2.7]{Duverney2024},
	\[
	\begin{aligned}
		\frac{\df^2}{\df x^2} \log \omega_n(x) &= n \left[\frac{c^2e^{cx}}{(1+e^{cx})^2} -n\psi^{(1)}(n-nx+1) - n\psi^{(1)}(nx+1)\right]\\
		& \leq n \left[\frac{c^2e^{cx}}{(1+e^{cx})^2} -\frac{1}{1-x+n^{-1}} - \frac{1}{x+n^{-1}}\right] . 
	\end{aligned}
	\]
	For $x \in [0,1]$ and $y > 0$, let
	\[
	h(x,y) = \frac{c^2e^{cx}}{(1+e^{cx})^2} -\frac{1}{1-x+y} - \frac{1}{x+y}.
	\]
	It suffices to show that $h(x,1/n) < 800/n$ for $x \in [0,1]$ when $n > 20$.
	
	Assume that $n > 20$.
	We study two cases.
	
	\noindent{\bf Case 1:} $x \in [0,0.05] \cup [0.95,1]$.
	It can be verified that $x \mapsto c^2 e^{cx}/(1+e^{cx})^2$ is decreasing, so it holds that
	\[
	h(x,1/n) < \frac{c^2e^{0}}{(1+e^{0})^2} - \frac{1}{1+1/20} - \frac{1}{0.05+1/20} < 0.
	\]
	
	\noindent{\bf Case 2:} $x \in (0.05, 0.95)$.
	Lemma \ref{lem:n-infty} implies that $h(x, 0) \leq 0$.
	Then
	\[
	\begin{aligned}
		h(x,1/n) &\leq \sup_{y \in [0,1/20]} \left| \frac{\partial h(x,y)}{\partial y} \right| \frac{1}{n} \\
		&=  \sup_{y \in [0,1/20]} \left| \frac{1}{(x+y)^2} + \frac{1}{(1-x+y)^2} \right| \frac{1}{n} \\
		&\leq 800/n.
	\end{aligned}
	\]
	 
	 The proof is then complete.
\end{proof}

\end{document}